%% file: main.tex
\documentclass[letterpaper]{article} % DO NOT CHANGE THIS
\usepackage{aaai25}  % DO NOT CHANGE THIS
\usepackage{times}  % DO NOT CHANGE THIS
\usepackage{helvet}  % DO NOT CHANGE THIS
\usepackage{courier}  % DO NOT CHANGE THIS
\usepackage[hyphens]{url}  % DO NOT CHANGE THIS
\usepackage{graphicx} % DO NOT CHANGE THIS
\usepackage{natbib}  % DO NOT CHANGE THIS AND DO NOT ADD ANY OPTIONS TO IT
\usepackage{caption} % DO NOT CHANGE THIS AND DO NOT ADD ANY OPTIONS TO IT
\usepackage{algorithm}
\usepackage{algorithmic}

\usepackage{newfloat}
\usepackage{listings}
\DeclareCaptionStyle{ruled}{labelfont=normalfont,labelsep=colon,strut=off} % DO NOT CHANGE THIS
\floatstyle{ruled}
\newfloat{listing}{tb}{lst}{}
\floatname{listing}{Listing}
\usepackage{multicol}
\usepackage{multirow}
\usepackage{graphicx}
\usepackage{xcolor}
\usepackage{algorithm}
\usepackage{algorithmic}
\usepackage{subcaption} 
\usepackage{amsmath}
\usepackage{amssymb}
\usepackage{amsthm}
\usepackage{rotating}
\newtheorem{theorem}{Theorem}

\newtheorem{lemma}{Lemma}

\title{Task-Specific Preconditioner for Cross-Domain Few-Shot Learning}
\author{
    Suhyun Kang\textsuperscript{\rm 1}, Jungwon Park\textsuperscript{\rm 2}, Wonseok Lee\textsuperscript{\rm 3}, Wonjong Rhee\textsuperscript{\rm 2,3}\thanks{Corresponding author}
}
\affiliations{
    \textsuperscript{\rm 1} Samsung Research, Seoul, South Korea\\
    \textsuperscript{\rm 2} Department of Intelligence and Information, Seoul National University, Seoul, South Korea\\
    \textsuperscript{\rm 3} IPAI, Seoul National University, Seoul, South Korea\\
    su1019.kang@samsung.com; \{quoded97, dnjstjr1017, wrhee\}@snu.ac.kr
}

\begin{document}

\maketitle

\input{Latex_files/0_abstract}
\input{Latex_files/1_introduction}
\input{Latex_files/2_related_works}
\input{Latex_files/3_backgrounds}
\input{Latex_files/4_method}

\input{Latex_files/5_experiments}

\input{Latex_files/6_discussion}
\input{Latex_files/7_conclusion}
\input{Latex_files/8_acknowledgement}

\bibliography{aaai25}
\clearpage
\input{Latex_files/9_appendix}
\end{document}

%% file: Latex_files/0_abstract.tex
\begin{abstract}
  Cross-Domain Few-Shot Learning~(CDFSL) methods typically parameterize models with task-agnostic and task-specific parameters. To adapt task-specific parameters, recent approaches have utilized fixed optimization strategies, despite their potential sub-optimality across varying domains or target tasks. To address this issue, we propose a novel adaptation mechanism called Task-Specific Preconditioned gradient descent~(TSP). Our method first meta-learns Domain-Specific Preconditioners~(DSPs) that capture the characteristics of each meta-training domain, which are then linearly combined using task-coefficients to form the Task-Specific Preconditioner. The preconditioner is applied to gradient descent, making the optimization adaptive to the target task. We constrain our preconditioners to be positive definite, guiding the preconditioned gradient toward the direction of steepest descent. Empirical evaluations on the Meta-Dataset show that TSP achieves state-of-the-art performance across diverse experimental scenarios. 
  % \keywords{Cross-Domain Few-Shot Classification \and Meta-Learning \and Preconditioned Gradient Descent}
\end{abstract}

%% file: Latex_files/1_introduction.tex
\section{Introduction}
\label{sec:intro}
Few-Shot Learning~(FSL) aims to learn a model that can generalize to novel classes using a few labeled examples. Recent advancements in FSL have been significantly propelled by meta-learning methods \cite{snell2017prototypical, finn2017model, sung2018learning, oreshkin2018tadam, garnelo2018conditional, rajeswaran2019meta}. 
These approaches have achieved outstanding results in single domain FSL benchmarks such as Omniglot~\cite{lake2011one} and \textit{mini}Imagenet~\cite{ravi2016optimization}. However, recent studies~\cite{chen2019closer,tian2020rethinking} have revealed that many existing FSL methods struggle to generalize in cross-domain setting, where the test data originates from domains that are either unknown or previously unseen. To study the challenge of generalization in cross-domain few-shot tasks, \citet{triantafillou2019meta} introduced the \textit{Meta-Dataset}, a more realistic, large-scale, and diverse benchmark. It includes multiple datasets from a variety of domains for both meta-training and meta-testing phases.

Leveraging the Meta-Dataset, various Cross-Domain Few-Shot Learning (CDFSL) methods have been developed~\cite{requeima2019fast,bateni2020improved, bateni2022enhancing,liu2021multi,triantafillou2021learning,li2021universal,li2022cross,dvornik2020selecting,liu2020universal,guo2023task,tian2024mokd}, demonstrating significant advancements in this field. 
These approaches typically parameterize deep neural networks with a large set of task-agnostic parameters alongside a smaller set of task-specific parameters. Task-specific parameters are optimized to the target task through an adaptation mechanism, generally following one of two primary methodologies. 
The first approach utilizes an auxiliary network functioning as a parameter generator, which, upon receiving a few labeled examples from the target task, outputs optimized task-specific parameters~\cite{requeima2019fast,bateni2020improved, bateni2022enhancing,liu2020universal,liu2021multi}. 
The second approach directly fine-tunes the task-specific parameters through gradient descent using a few labeled examples from the target task~\cite{dvornik2020selecting,li2021universal,triantafillou2021learning,li2022cross, tian2024mokd}. 

\begin{figure*}
    \centering
    \begin{minipage}{0.65\textwidth}
        \centering
        \begin{subfigure}[t]{0.54\textwidth}
            \includegraphics[width=\textwidth]{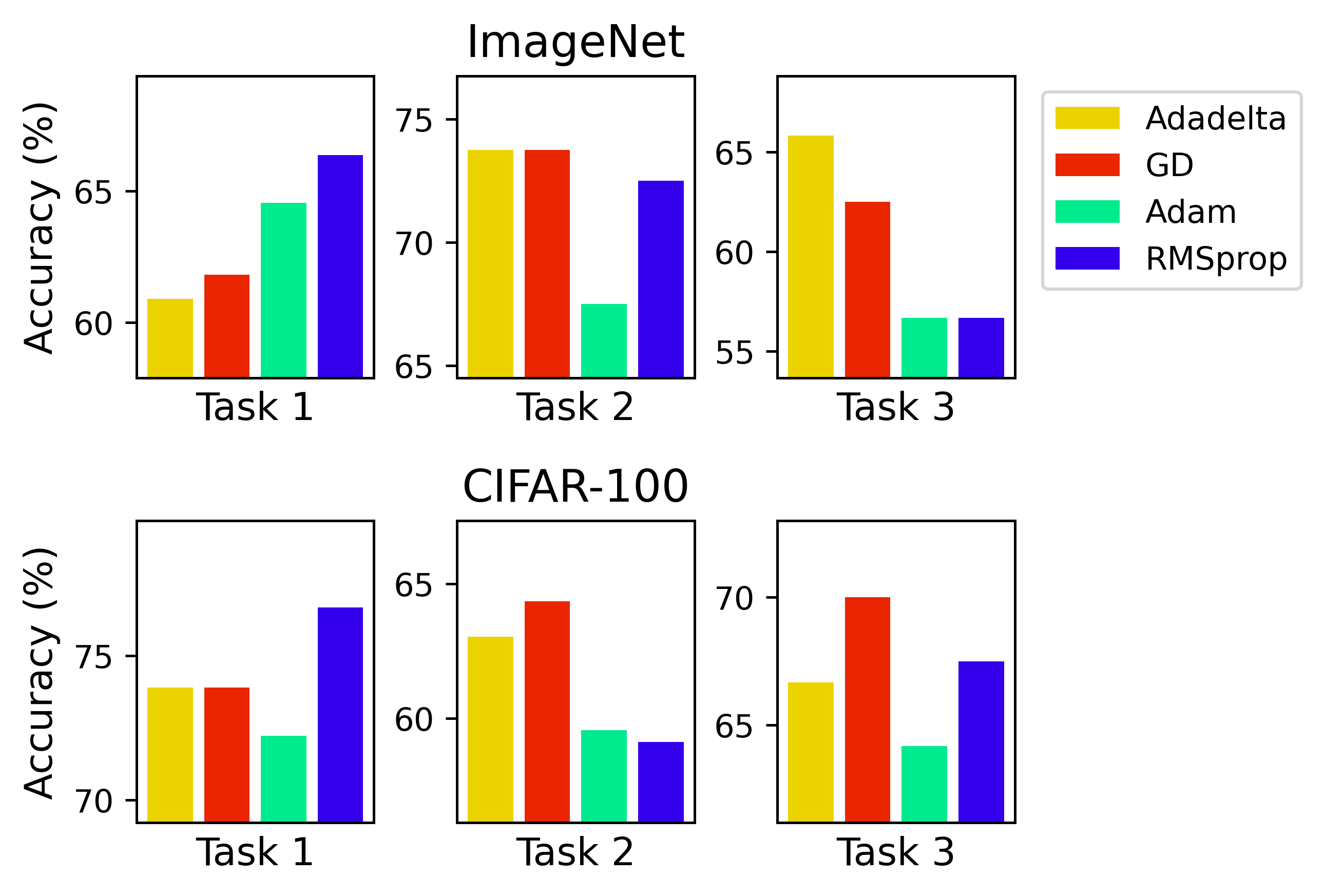}
            \subcaption[]{}
            \label{subfig:motivation1}
        \end{subfigure}
        \begin{subfigure}[t]{0.4\textwidth}
            \includegraphics[width=\textwidth]{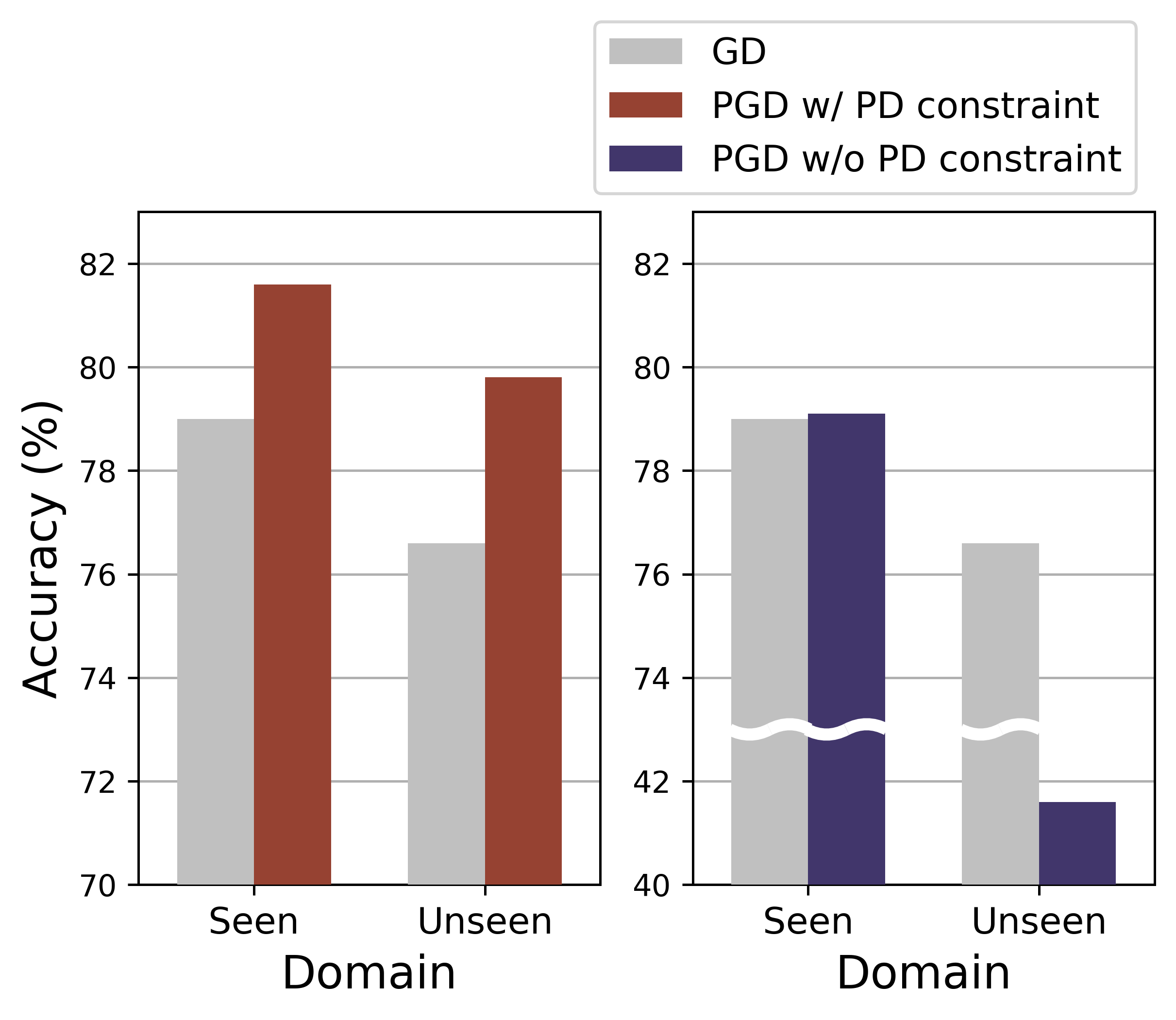}
            \subcaption[]{}
            \label{subfig:motivation2}
        \end{subfigure}
        \caption{All experiments are conducted baed on TSA. (a) The optimal optimization strategy can vary significantly depending on the nature of the target task, leading to notable differences in performance on the Meta-Dataset. (b) The accuracy of seen and unseen for the Meta-Dataset. Compared to the baseline of using gradient descent, adopting a preconditioner without a PD constraint can be unreliable. With a PD constraint, it becomes reliable to adapt the preconditioner to the target task. Further details on these preconditioners are provided in Appendix A.} % ~\ref{sec:three_type}
    \end{minipage}
    \hfill
    \begin{minipage}{0.3\textwidth}
        \includegraphics[width=\textwidth]{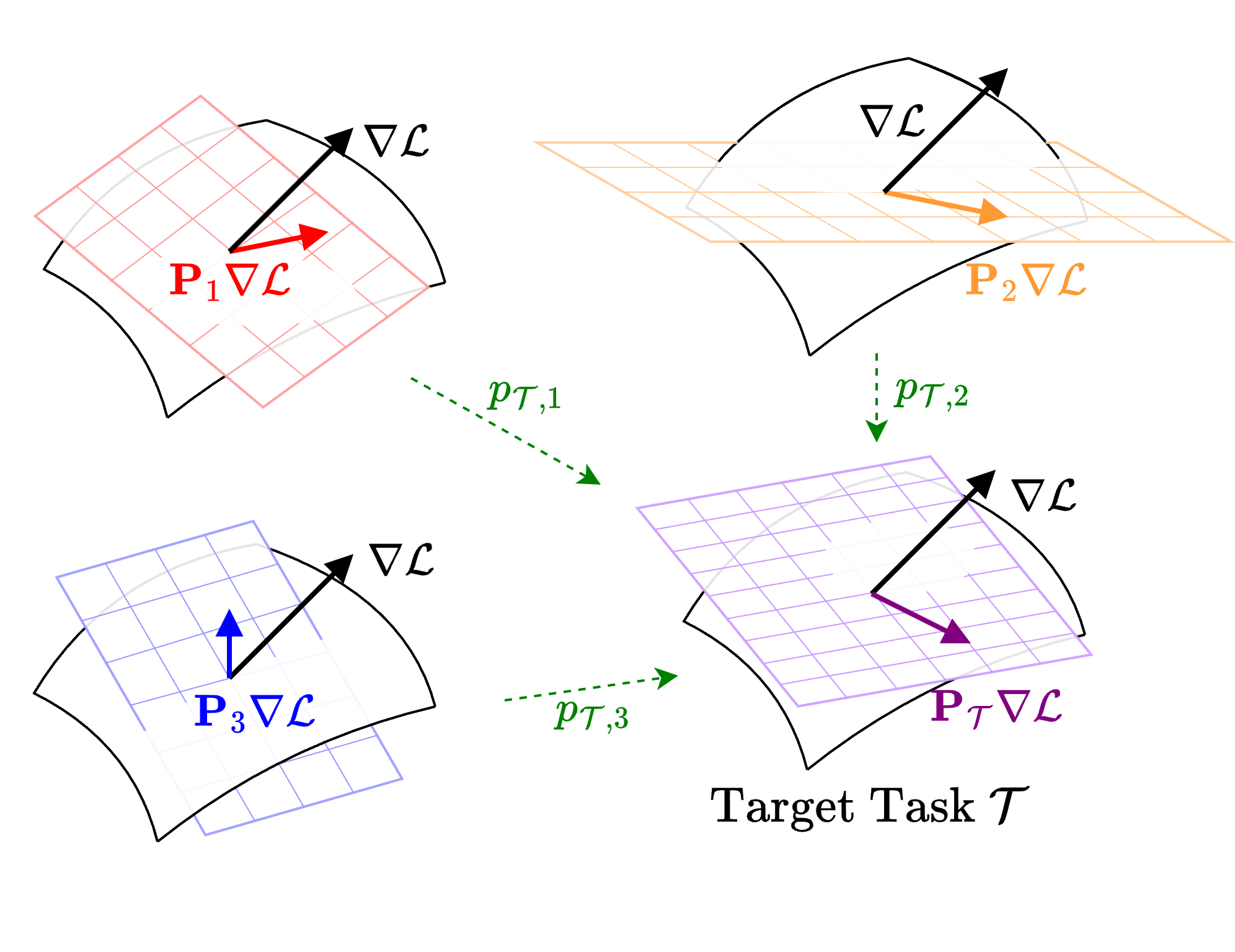}
        \caption{Illustration of forming a Task-Specific Preconditioner based on three DSPs that have been meta-trained for three meta-training domains.}
        \label{subfig:TSP_illustration}
    \end{minipage}
\end{figure*}

While both approaches have improved CDFSL performance through adaptation mechanism, a common limitation persists in the optimization strategies employed by these methods. Specifically, both approaches employ a fixed optimization strategy across different target tasks. 
However, Figure~\ref{subfig:motivation1} shows that the optimal choice of optimizer may vary significantly depending on the given domain or target task. This implies that the performance can be significantly improved by adapting an optimization strategy to align well with the target domain and task. However, devising an effective and reliable scheme for its implementation has been challenging. 

One promising approach for establishing a robust adaptive optimization scheme is to leverage Preconditioned Gradient Descent (PGD)~\cite{himmelblau2018applied}. PGD operates by specifying a preconditioning matrix, often referred to as a \textit{preconditioner}, which re-scales the geometry of the parameter space. 
In the field of machine learning, previous research has shown that if the preconditioner is positive definite~(PD), it establishes a valid Riemannian metric, which represents the geometric characteristics~(e.g., curvature) of the parameter space and steers preconditioned gradients in the direction of steepest descent~\cite{amari1967theory, amari1996neural, amari1998natural1, amari1998natural2}. 
While the effectiveness of positive definiteness in PGD is supported by existing theoretical findings, its efficacy as an adaptive optimization scheme in CDFSL can be examined through a simple comparison. 
In Figure~\ref{subfig:motivation2}, we compare PGD with and without a PD constraint for the preconditioner on the Meta-Dataset. Without a PD constraint, PGD shows markedly inferior performance, especially in unseen domains. Conversely, with a PD constraint, PGD consistently exhibits performance improvements across seen and unseen domains compared to the baseline using GD. This supports the pivotal role of positive definiteness in PGD for CDFSL. 

Inspired by these findings, we introduce a novel adaptation mechanism named Task-Specific Preconditioned gradient descent~(TSP). 
In our approach, we establish a Task-Specific Preconditioner that is constrained to be positive definite and adapt it to the specific nature of the target task. 
This preconditioner consists of two components. 
The first component is the Domain-Specific Preconditioners~(DSPs), which are uniquely defined for each meta-training domain and meta-trained on tasks sampled from these domains through bi-level optimization during the meta-training phase. 
The second component is task-coefficient, which approximates the compatibility between the target task and each meta-training domain. Figure~\ref{subfig:TSP_illustration} illustrates the construction of the Task-Specific Preconditioner. For a given target task ${\mathcal{T}}$, the Task-Specific Preconditioner $\mathbf{P}_{\mathcal{T}}$ is constructed by linearly combining the DSPs ${\mathbf{P}_k}$ from multiple seen domains, with each weighted by the corresponding task-coefficient ${p_{\mathcal{T},k}}$. This process produces a preconditioner specifically adapted to the geometric characteristics of the target task’s parameter space. By integrating knowledge from multiple seen domains, TSP distinguishes itself from traditional PGD techniques, such as GAP~\cite{kang2023meta}, which are discussed further in Section~\ref{sec:discussion}. 
% \textcolor{red}{As depicted in Figure~\ref{subfig:TSP_illustration}, when a target task ${\mathcal{T}}$ is provided, the Task-Specific Preconditioner $\mathbf{P}_{\mathcal{T}}$ is formed by linearly combining the meta-trained DSPs ${\mathbf{P}_k}$ from multiple domains, weighted by the task-coefficient ${p_{\mathcal{T},k}}$. 
% This approach allows the preconditioner to effectively adapt to the target task by leveraging the diverse domain knowledge embedded in the DSPs, thereby improving task performance.} 
% As depicted in Figure~\ref{subfig:TSP_illustration}, when a target task ${\mathcal{T}}$ is provided, the Task-Specific Preconditioner $\mathbf{P}_{\mathcal{T}}$ is formed by linearly combining the meta-trained DSPs $\{\mathbf{P}_k\}$ with the task coefficients $\{p_{\mathcal{T},k}\}$. 
% This combination enables the preconditioner 
% This results in a preconditioner specifically tailored to the geometric characteristics of the parameter space for the target task. 
Applying our approach to state-of-the-art CDFSL methods, such as TSA or TA$^2$-Net, significantly enhances performance on Meta-Dataset. 
For example, in multi-domain settings, applying TSP to TA$^2$-Net~\cite{guo2023task} achieves the best performance across all datasets. 

%% file: Latex_files/2_related_works.tex
\section{Related Works}
\label{sec:related_works}
% \subsection{Meta-Learning for Few-Shot Learning}
\paragraph{Meta-Learning for Few-Shot Learning}
Until recently, numerous approaches in the field of few-shot learning have adopted the meta-learning framework. 
% Until recently, numerous approaches in the field of few-shot classification have adopted the meta-learning framework~\cite{garcia2017few, sung2018learning, snell2017prototypical, oreshkin2018tadam, santoro2016meta, munkhdalai2017meta, mishra2017simple, garnelo2018conditional, ravi2016optimization, finn2017model, yoon2018bayesian, rajeswaran2019meta}. 
These approaches can be mainly divided into three types: metric-based, model-based, and optimization-based methods. Metric-based methods~\cite{garcia2017few, sung2018learning, snell2017prototypical, oreshkin2018tadam} train a feature encoder to extract features from support and query samples. They employ a nearest neighbor classifier with various distance functions to calculate similarity scores for predicting the labels of query samples. 
Model-based methods~\cite{santoro2016meta, munkhdalai2017meta, mishra2017simple, garnelo2018conditional} train an encoder to generate task-specific models from a few support samples. Optimization-based methods~\cite{ravi2016optimization, finn2017model, yoon2018bayesian, rajeswaran2019meta} train a model that can quickly adapt to new tasks with a few support samples, employing a bi-level optimization.  
% comprising inner and outer-level optimizations. 
In our method, we employ the bi-level optimization used in the optimization-based methods. 

% \subsection{Cross-Domain Few-Shot Learning~(CDFSL)}
\paragraph{Cross-Domain Few-Shot Learning~(CDFSL)}
% The primary objective of Cross-Domain Few-shot Classification~(CDFSC) is to obtain a universal model capable of effectively classifying new classes from a wide range of domains. This universal model is typically trained across multiple training domains and is applied in cross-domain classification, with a few support images from the target task. 
Recent CDFSL methods define the universal model as a deep neural network and partition it into task-agnostic and task-specific parameters.  
% Recent CDFSC methods~\cite{requeima2019fast, bateni2020improved, bateni2022enhancing,liu2021multi,triantafillou2021learning,li2021universal,li2022cross,guo2023task,dvornik2020selecting,liu2020universal} define the universal model as a deep neural network and partition it into task-agnostic and task-specific weights. 
% 
The task-agnostic parameters represent generic characteristics that are valid for a range of tasks from various domains. On the other hand, the task-specific parameters represent adaptable attributes that are optimized to the target tasks through an adaptation mechanism. 
Task-agnostic parameters can be designed as a single network or multiple networks. 
The single network is trained on a large dataset from single domain~\cite{requeima2019fast, bateni2020improved, bateni2022enhancing, liu2021multi} or multiple domains~\cite{triantafillou2021learning, li2021universal, li2022cross, guo2023task}, whereas the multiple networks are trained individually on each domain~\cite{dvornik2020selecting, liu2020universal}. Task-specific parameters can be designed as selection parameters~\cite{dvornik2020selecting, liu2020universal}, pre-classifier transformation~\cite{li2021universal, li2022cross, guo2023task}, Feature-wise Linear Modulate~(FiLM) layer~\cite{requeima2019fast, bateni2020improved, bateni2022enhancing, liu2021multi, triantafillou2021learning}, or Residual Adapter~(RA)~\cite{li2022cross, guo2023task}. As the adaptation mechanism for the task-specific parameters, several studies~\cite{requeima2019fast,bateni2020improved, bateni2022enhancing, liu2020universal, liu2021multi} meta-learn an auxiliary network, which generates task-specific parameters adapted to the target task. On the other hand, other studies~\cite{dvornik2020selecting, li2021universal, triantafillou2021learning, li2022cross} employ gradient descent to adapt task-specific parameters to the target task. In our work, we propose a novel adaptation mechanism in the form of a task-specific optimizer, which adapts task-specific parameters to the target task. 

% \subsection{Preconditioned Gradient Descent in Meta-Learning}
\paragraph{Preconditioned Gradient Descent in Meta-Learning}
In meta-learning, several optimization-based approaches~\cite{li2017meta, lee2018gradient, park2019meta, rajasegaran2020meta, simon2020modulating, zhao2020meta, von2021learning, kang2023meta} have incorporated Preconditioned Gradient Descent (PGD) to adapt network's parameters to the target task~(i.e., inner-level optimization). They meta-learn a preconditioning matrix, called a preconditioner, which is utilized to precondition the gradient. The preconditioner was kept static in most of the previous works~\cite{li2017meta, lee2018gradient, park2019meta,zhao2020meta,von2021learning}. Several prior studies have devised preconditioners tailored to adapt either per inner step~\cite{rajasegaran2020meta}, per task~\cite{simon2020modulating}, or both simultaneously~\cite{kang2023meta}. 
Motivated by previous works~\cite{amari1967theory, amari1996neural, amari1998natural1,kakade2001natural,amari1998natural2}, \cite{kang2023meta} recently investigated the constraint of the preconditioner to satisfy the condition for a Riemannian metric (i.e., positive definiteness). They demonstrated that enforcing this constraint on the preconditioner was essential for improving the performance in few-shot learning. In our study, we propose a novel preconditioned gradient descent method with meta-learned task-specific preconditioner that guarantees positive definiteness for improving performance in CDFSL. 

%% file: Latex_files/3_backgrounds.tex
\section{Backgrounds}
\label{sec:backgrounds}
\paragraph{Task Formulation for Meta-Learning in CDFSL}
In CDFSL, task $\mathcal{T}$ is formulated differently compared to traditional few-shot learning. In traditional few-shot learning, tasks are sampled from a single domain, resulting in the same form in both meta-training and meta-testing:
\begin{equation}
    \text{meta-training and meta-testing: }\mathcal{T}=\{\mathcal{S}_{\mathcal{T}}, \mathcal{Q}_{\mathcal{T}}\}
\end{equation}
where $\mathcal{S}_{\mathcal{T}}$ is a support set and $\mathcal{Q}_{\mathcal{T}}$ is a query set. 
On the other hand, in CDFSL, tasks are sampled from multiple domains, leading to different forms in meta-training and meta-testing:
\begin{equation}
    \begin{split}
        & \text{meta-training: }\mathcal{T}=\{\mathcal{S}_{\mathcal{T}} \mathcal{Q}_{\mathcal{T}},d_{\mathcal{T}}\},\\
        & \text{meta-testing: }\mathcal{T}=\{\mathcal{S}_{\mathcal{T}} \mathcal{Q}_{\mathcal{T}}\},
    \end{split}
\end{equation}
where $d_{\mathcal{T}}$ is a domain label indicating the domain from which the task was sampled. For instance, the domain label is an integer between $1$ and $K$ for $K$ domains~(i.e., $1\leq d_{\mathcal{T}}\leq K$). 

\paragraph{Bi-level Optimization in Meta-Learning}
Bi-level optimization~\cite{rajeswaran2019meta} consists of two levels of main optimization processes: inner-level and outer-level optimizations. 
Let $f_{\theta(\phi)}$ be a model, where the parameter $\theta(\phi)$ is parameterized by the meta-parameter $\phi$. For a task $\mathcal{T}=\{\mathcal{S}_{\mathcal{T}},\mathcal{Q}_{\mathcal{T}}\}$, the inner-level optimization is defined as:
\begin{equation}
\label{eqn:definition_inner_opt}
    \begin{split}
        \theta_{\mathcal{T},T}(\phi)=\theta_{\mathcal{T}, 0}(\phi)-\alpha_{\text{in}}\cdot\sum_{t=0}^{T-1}\nabla_{\theta}\mathcal{L}_{\text{in}}(\theta_{\mathcal{T},t}(\phi);\mathcal{S}_{\mathcal{T}})
    \end{split}
\end{equation}
where $\theta_{\mathcal{T},0}(\phi)=\theta(\phi)$, $\alpha_{\text{in}}$ is the learning rate for the inner-level optimization, $\mathcal{L}_{\text{in}}$ is the inner-level’s loss function, and $T$ is the total number of gradient descent steps.
% $\theta_{\mathcal{T}, t}$ is task-specific weights for task $\mathcal{T}$ at $t$-step
With $\mathcal{Q}_{\mathcal{T}}$ in each task, we can define outer-level optimization as:
\begin{equation}
    \phi\leftarrow\phi-\alpha_{\text{out}}\cdot\nabla_{\phi}\mathbb{E}_{\mathcal{T}}\Big[\mathcal{L}_{\text{out}}(\theta_{\mathcal{T}, T}(\phi);\mathcal{Q}_{\mathcal{T}})\Big]
\end{equation}
where $\alpha_{\text{out}}$ is the learning rate for the outer-level optimization, and $\mathcal{L}_{\text{out}}$ is the outer-level’s loss function. 

\paragraph{Preconditioned Gradient Descent~(PGD)}
PGD is a technique that minimizes empirical risk by using a gradient update with a preconditioner that re-scales the geometry of the parameter space. 
% PGD is a technique employed to reduce empirical risk by utilizing a gradient update method, complemented by a preconditioner that modifies the geometric properties of the parameter space.
Given model parameters $\theta$ and task $\mathcal{T} = \{\mathcal{S}_{\mathcal{T}}, \mathcal{Q}_{\mathcal{T}}\}$, we can formally define the preconditioned gradient descent with a preconditioner $\mathbf{P}$ as follows:
\begin{equation}
    \label{eqn:pgd}
    \theta_{\mathcal{T},t}=\theta_{\mathcal{T},t-1}-\alpha\cdot\mathbf{P}\nabla_{\theta}\mathcal{L}(\theta_{\mathcal{T},t-1};\mathcal{S}_{\mathcal{T}}),\,\,\,\,t=1,\cdots
\end{equation}
where $\theta_{\mathcal{T},0}=\theta$, $\mathcal{L}(\theta_{\mathcal{T},t};\mathcal{S}_{\mathcal{T}})$ is the empirical loss associated with the task $\mathcal{T}$, and $\theta_{\mathcal{T},t}$ is the parameters. 
When the preconditioner $\mathbf{P}$ is chosen to be the identity matrix $\mathbf{I}$, Eq.~(\ref{eqn:pgd}) becomes the standard Gradient Descent~(GD). 
The choice of $\mathbf{P}$ to leverage second-order information offers several options, including the inverse Fisher information matrix $\mathbf{F}^{-1}$, leading to the Natural Gradient Descent~(NGD)~\cite{amari1998natural1}, the inverse Hessian matrix $\mathbf{H}^{-1}$, corresponding to Newton's method~\cite{lecun2002efficient}, and the diagonal matrix estimation with the past gradients, which results in adaptive gradient methods~\cite{duchi2011adaptive, kingma2014adam}. They often reduce the effect of pathological curvature and speed up the optimization~\cite{amari2020does}. 

\paragraph{Dataset Classifier} 
In CDFSL, Dataset Classifier~\cite{triantafillou2021learning} reads a support set in a few-shot task and predicts from which of the training datasets it was sampled. Formally, let $\mathcal{T}=\{\mathcal{S}_{\mathcal{T}}, \mathcal{Q}_{\mathcal{T}}, d_{\mathcal{T}}\}$ be a train task sampled from $K$ domains. Let $g$ be a dataset classifier that takes the support set $\mathcal{S}_{\mathcal{T}}$ as input and generates logits as follows:
\begin{equation}
    g(\mathcal{S}_{\mathcal{T}})= z_{\mathcal{T}}=(z_{\mathcal{T},1},\cdots,z_{\mathcal{T},K}) \in \mathbb{R}^{K}
\end{equation} 
In~\cite{triantafillou2021learning}, the dataset classifier $g$ is trained to minimize the cross-entropy loss for the dataset classification problem~(i.e., classification problem with $K$ classes). 

%% file: Latex_files/4_method.tex
\section{Method}
\begin{figure}[t]
    \centering
    \begin{subfigure}{0.48\textwidth}
        \includegraphics[width=\textwidth]{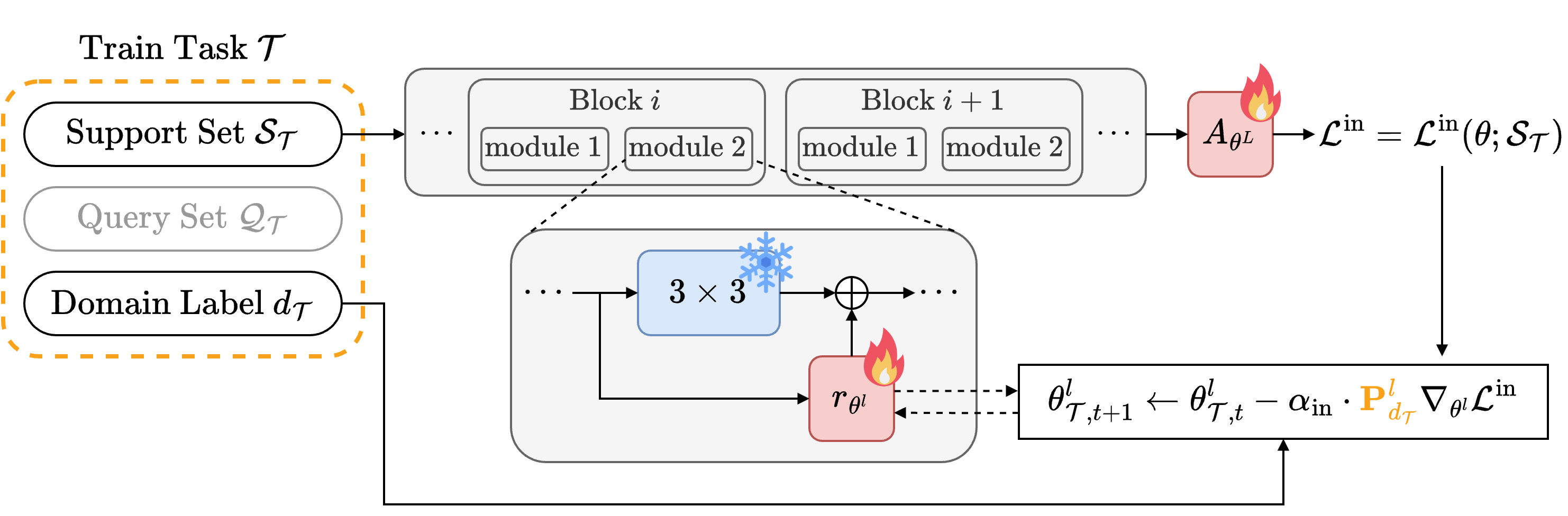}
        \caption{}
        \label{subfig:tsp_train}
    \end{subfigure}
    \begin{subfigure}{0.48\textwidth}
        \includegraphics[width=\textwidth]{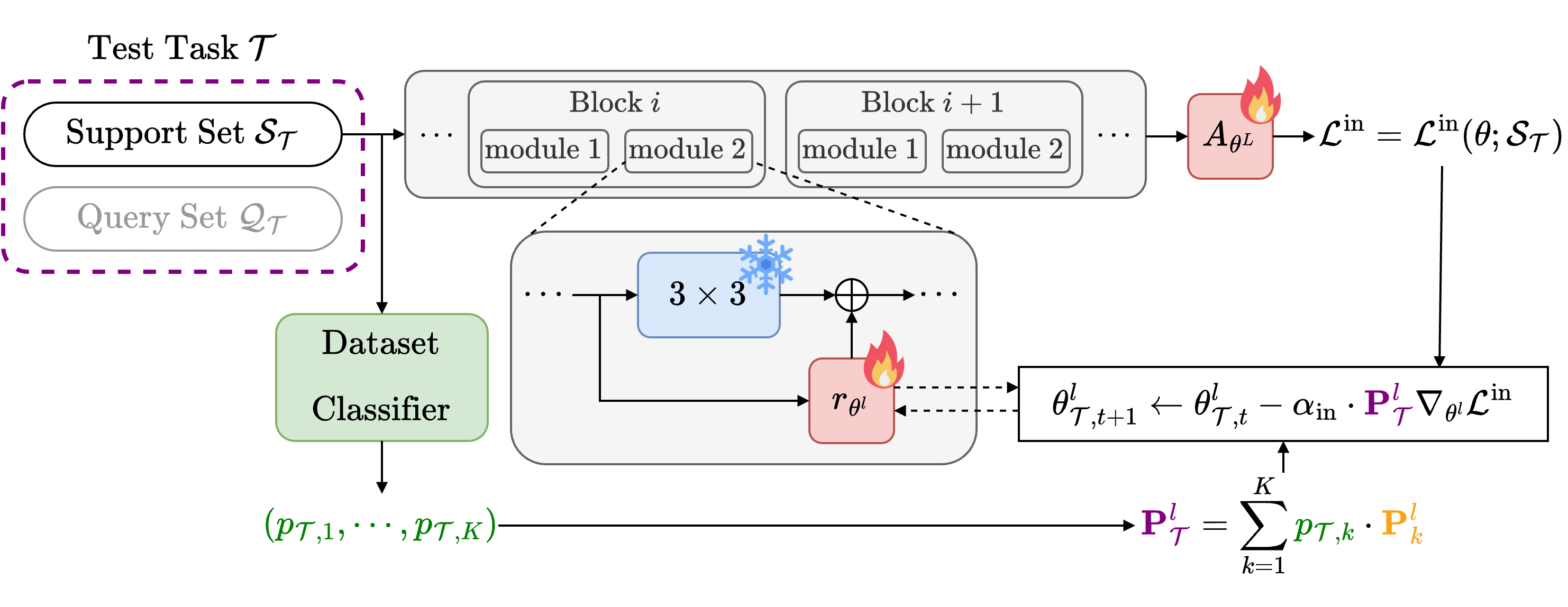}
        \caption{}
        \label{subfig:tsp_test}
    \end{subfigure}
    \caption{
    (a) PGD with Domain-Specific Preconditioner~(DSP) in the inner-level optimization. During meta-training, for a train task $\mathcal{T}$, DSP is chosen based on the domain label $d_{\mathcal{T}}$, and each task-specific parameter $\theta^l$ are optimized using PGD with the selected DSP $\mathbf{P}^l_{d_{\mathcal{T}}}$. 
    (b) PGD with Task-Specific Preconditioner. During meta-testing, for a test task, each Task-Specific Preconditioner $\mathbf{P}^l_{\mathcal{T}}$ is contructed using DSPs and task-coefficients generated by Dataset Classifier. Each task-specific parameter $\theta^l$ is then then optimized using PGD with $\mathbf{P}^l_{\mathcal{T}}$. 
    }
    \label{fig:tsp}
\end{figure}

In this section, we propose a novel adaptation mechanism named Task-Specific Preconditioned gradient descent~(TSP). 
We first introduce Domain-Specific Preconditioner~(DSP) and task-coefficients. Then, we describe the construction of Task-Specific Preconditioner using DSP and task-coefficients. Lastly, we show the positive definiteness of Task-Specific Preconditioner, which establishes it as a valid Riemannian metric. The algorithm for the training and testing procedures is provided in Appendix B.%~\ref{sec:pseudocode}. 
% We introduce Domain-Specific Preconditioner~(DSP) in~\ref{subsec:dsp}, and \textit{task-coefficients} in~\ref{subsec:tc}. Then, in~\ref{subsec:tsp}, we describe the construction of Task-Specific Preconditioner using DSP and task-coefficients. Lastly, in~\ref{subsec:pd of tsp}, we elucidate the positive definiteness of Task-Specific Preconditioner, which establishes it as a valid Riemannian metric. 
\subsection{Domain-Specific Preconditioner~(DSP)}
\label{subsec:dsp}
Consider $L$ task-specific parameters $\theta=\{\theta^l\in\mathbb{R}^{m_l\times m_l}\}_{l=1}^L$. For $K$ domains, we first define meta-parameters $\mathcal{M}_1, \cdots, \mathcal{M}_K$ as follows:
\begin{equation}
    \mathcal{M}_k=\{\mathbf{M}^l_{k}\in\mathbb{R}^{m_l\times m_l}\}_{l=1}^L,\,\,\,\,k=1,\cdots,K
\end{equation}
Then, for all $l$, we define Domain-Specific Preconditioners~(DSPs) $\mathbf{P}^l_k$ using the meta-parameters as follows:
\begin{equation}
\label{eqn:definition_of_dsp}
    \mathbf{P}^l_{k}=\mathbf{M}_{k}^{l\mathbf{T}} \mathbf{M}^l_{k} + \mathbf{I},\,\,\,\,k=1,\cdots,K
\end{equation}
We compare various DSP designs~(See Table~\ref{tab:Ablation-pd preconditioners}) in Section~\ref{sec:ablation} and choose the form of Eq.~(\ref{eqn:definition_of_dsp}). 
% In Ablation Studies of Experiments, we further explore various DSP design choices.   
Through bi-level optimization, DSPs can be meta-learned as follows. 
 % , we explain how to meta-learn DSPs through bi-level optimization. 

\subsubsection{Inner-level Optimization}
For each train task $\mathcal{T}=\{\mathcal{S}_{\mathcal{T}}, \mathcal{Q}_{\mathcal{T}},d_{\mathcal{T}}\}$, in the inner-level optimization, we optimize the task-specific parameters $\theta$ through preconditioned gradient descent using $\mathbf{P}^l_{d_{\mathcal{T}}}$, updating $\theta$ as follows:
\begin{equation}
\label{eqn:dsp}
    \theta^l_{\mathcal{T}, T}=\theta^l_{\mathcal{T},0} - \alpha_{\text{in}}\cdot\sum^{T-1}_{t=0}\mathbf{P}^l_{d_{\mathcal{T}}}\nabla_{\theta^l_{\mathcal{T},t}}\mathcal{L}_{\text{in}}(\theta_{\mathcal{T},t};\mathcal{S}_{\mathcal{T}}),
\end{equation}
where $\theta^l_{\mathcal{T}, 0}=\theta^l$, $\alpha_{\text{in}}$ is the learning rate for the inner-level optimization, $T$ is the total number of gradient descent steps, and $\mathcal{L}_{\text{in}}$ is the inner-level's loss function. 

\subsubsection{Outer-level Optimization}
In the outer-level optimization, we meta-learn meta-parameters $\mathcal{M}_1,\cdots \mathcal{M}_K$ as follows:
\begin{equation}
\label{eqn:definition_outer_opt}
    \mathcal{M}_k \leftarrow \mathcal{M}_k - \alpha_{\text{out}}\cdot\nabla_{\mathcal{M}_k}\mathbb{E}_{\mathcal{T}}\Big[\mathcal{L}_{\text{out}}(\theta_{\mathcal{T},T};\mathcal{Q}_{\mathcal{T}})\Big],\,\,\,\,k=1,\cdots,K
\end{equation}
where $\alpha_{\text{out}}$ is the learning rate for outer-level optimization and $\mathcal{L}_{\text{out}}$ is the outer-level’s loss function.

\subsection{Task-coefficients}
\label{subsec:tc}
Consider the dataset classifier $g$. 
% Consider the dataset classifier $g$ defined in \ref{sec:dataset_classifier}. 
Given a train task $\mathcal{T}=\{\mathcal{S}_{\mathcal{T}},\mathcal{Q}_{\mathcal{T}},d_{\mathcal{T}}\}$, we define task-coefficients $p_{\mathcal{T}, 1}\cdots,p_{\mathcal{T}, K}$ as follows:
\begin{equation}
\label{eqn:task_coefficients}
    (p_{\mathcal{T}, 1},\cdots, p_{\mathcal{T},K}) = \text{Softmax}(z_{\mathcal{T},1},\cdots,z_{\mathcal{T},K})
\end{equation}
where $g(\mathcal{S}_{\mathcal{T}})=(z_{\mathcal{T},1},\cdots,z_{\mathcal{T},K})$. 
Note that we use the sigmoid function instead of softmax in the single-domain setting because the output dimension of the dataset classifier is one. 
While \citet{triantafillou2021learning} updates the parameters of $g$ to minimize only the cross-entropy loss $\mathcal{L}_{\text{CE}}$ with respect to the dataset label $d_{\mathcal{T}}$, we train the dataset classifier $g$ to minimize the following augmented loss:
\begin{equation}
\label{eqn:loss_for_dataset_classifier}
    \mathcal{L}_{\text{CE}} + \lambda\cdot\mathcal{L}_{\text{Aux}}
\end{equation}
where $\lambda$ is a regularization parameter and $\mathcal{L}_{\text{Aux}}$ is the auxiliary loss, defined as follows:
\begin{equation}
    \mathcal{L}_{\text{Aux}}= \mathbb{E}_{\mathcal{T}}\Big[\mathcal{L}_{\text{out}}(\theta_{\mathcal{T},T};\mathcal{Q}_{\mathcal{T}})\Big]
\end{equation}
Here, task-specific parameters $\theta^l_{\mathcal{T}, T}$ can be obtained as follows:
\begin{equation}
\label{eqn:14}
    \theta^l_{\mathcal{T}, T}=\theta^l_{\mathcal{T},0} - \alpha_{\text{in}}\cdot\sum^{T-1}_{t=0}\sum^K_{k=1}p_{\mathcal{T},k} \cdot \mathbf{P}^l_k\nabla_{\theta^l_{\mathcal{T},t}}\mathcal{L}_{\text{in}}(\theta_{\mathcal{T},t};\mathcal{S}_{\mathcal{T}})
\end{equation}
where $\mathbf{P}^l_k$ is the $l$-th DSP of domain $k$.
In Eq.~(\ref{eqn:loss_for_dataset_classifier}), the cross-entropy loss guides the dataset classifier to prioritize the ground-truth domain of the support set. 
Concurrently, the auxiliary loss guides toward DSPs that minimize any adverse effects on the performance of the query set during the inner-level optimization. 

\begin{table*}[t]
\setlength{\tabcolsep}{1.5mm}
\centering
\caption{Performance comparison to state-of-the-art methods in a multi-domain setting. Mean accuracy and 95$\%$ confidence interval are reported. The best results are highlighted in \textbf{bold}.  TSP\textsuperscript{\dag} denotes TSP applied on TSA. TSP\textsuperscript{\dag\dag} denotes TSP applied on TA$^2$-Net.}
\label{tab:MDL results}
% \scalebox{0.63}{
\small
\begin{tabular}{ccccccccc|cc}
    \hline
    Test   Dataset   & SUR & URT & FLUTE & tri-M & URL  & TSA & TA$^2$-Net & MOKD & TSP\textsuperscript{\dag}  & TSP\textsuperscript{\dag\dag}   \\
    \hline
    ImageNet         & 56.2$\pm$1.0 & 56.8$\pm$1.1 & 58.6$\pm$1.0 & 51.8$\pm$1.1 & 58.8$\pm$1.1 & 59.5$\pm$1.0 & 59.6$\pm$1.0 & 57.3$\pm$1.1 & 60.5$\pm$1.0 & \textbf{60.7$\pm$1.0} \\
    Omniglot         & 94.1$\pm$0.4 & 94.2$\pm$0.4 & 92.0$\pm$0.6  & 93.2$\pm$0.5 & 94.5$\pm$0.4 & 94.9$\pm$0.4 & 95.5$\pm$0.4 & 94.2$\pm$0.5 & 95.6$\pm$0.4 & \textbf{96.0$\pm$0.4} \\
    Aircraft         & 85.5$\pm$0.5 & 85.8$\pm$0.5 & 82.8$\pm$0.7  & 87.2$\pm$0.5 & 89.4$\pm$0.4 & 89.9$\pm$0.4 & 90.5$\pm$0.4 & 88.4$\pm$0.5 & 90.5$\pm$0.4 & \textbf{91.2$\pm$0.4} \\
    Birds            & 71.0$\pm$1.0 & 76.2$\pm$0.8 & 75.3$\pm$0.8  & 79.2$\pm$0.8 & 80.7$\pm$0.8 & 81.1$\pm$0.8 & 81.4$\pm$0.8 & 80.4$\pm$0.8 & 82.3$\pm$0.7 & \textbf{82.5$\pm$0.7} \\
    Textures         & 71.0$\pm$0.8 & 71.6$\pm$0.7 & 71.2$\pm$0.8  & 68.8$\pm$0.8 & 77.2$\pm$0.7 & 77.5$\pm$0.7 & 77.4$\pm$0.7 & 76.5$\pm$0.7 & 78.6$\pm$0.6 & \textbf{79.1$\pm$0.6} \\
    Quick   Draw     & 81.8$\pm$0.6 & 82.4$\pm$0.6 & 77.3$\pm$0.7  & 79.5$\pm$0.7 & 82.5$\pm$0.6 & 81.7$\pm$0.6 & 82.5$\pm$0.6 & 82.2$\pm$0.6 & 83.0$\pm$0.7 & \textbf{83.2$\pm$0.6} \\
    Fungi            & 64.3$\pm$0.9 & 64.0$\pm$1.0 & 48.5$\pm$1.0  & 58.1$\pm$1.1 & 68.1$\pm$0.9 & 66.3$\pm$0.8 & 66.3$\pm$0.9 & 68.6$\pm$1.0 & 68.6$\pm$0.9 & \textbf{69.7$\pm$0.8} \\
    VGG Flower       & 82.9$\pm$0.8 & 87.9$\pm$0.6 & 90.5$\pm$0.5  & 91.6$\pm$0.6 & 92.0$\pm$0.5 & 92.2$\pm$0.5 & 92.6$\pm$0.4 & 92.5$\pm$0.5 & 93.3$\pm$0.4 & \textbf{93.4$\pm$0.4} \\ \hline
    Traffic   Sign   & 51.0$\pm$1.1 & 48.2$\pm$1.1 & 63.0$\pm$1.0  & 58.4$\pm$1.1 & 63.3$\pm$1.1 & 82.8$\pm$1.0 & 87.4$\pm$0.8 & 64.5$\pm$1.1 & 88.5$\pm$0.7 & \textbf{89.4$\pm$0.8} \\
    MSCOCO           & 52.0$\pm$1.1 & 51.5$\pm$1.1 & 52.8$\pm$1.1  & 50.0$\pm$1.0 & 57.3$\pm$1.0 & 57.6$\pm$1.0 & 57.9$\pm$0.9 & 55.5$\pm$1.0 & 58.5$\pm$0.9 & \textbf{59.8$\pm$0.9} \\
    MNIST            & 94.3$\pm$0.4 & 90.6$\pm$0.5 & 96.2$\pm$0.3  & 95.6$\pm$0.5 & 94.7$\pm$0.4 & 96.7$\pm$0.4 & 97.0$\pm$0.4 & 95.1$\pm$0.4 & \textbf{97.1$\pm$0.3} & \textbf{97.1$\pm$0.4} \\
    CIFAR-10         & 66.5$\pm$0.9 & 67.0$\pm$0.8 & 75.4$\pm$0.8  & 78.6$\pm$0.7 & 74.2$\pm$0.8 & 82.9$\pm$0.7 & 82.1$\pm$0.8 & 72.8$\pm$0.8 & 83.5$\pm$0.7 & \textbf{83.7$\pm$0.8} \\
    CIFAR-100        & 56.9$\pm$1.1 & 57.3$\pm$1.0 & 62.0$\pm$1.0  & 67.1$\pm$1.0 & 63.5$\pm$1.0 & 70.4$\pm$0.9 & 70.9$\pm$0.9 & 63.9$\pm$1.0 & 71.3$\pm$1.0 & \textbf{72.2$\pm$0.9} \\ \hline
    Avg   Seen   & 75.9 & 77.4 & 74.5  & 76.2  & 80.4 & 80.4 & 80.7 & 80.0 & 81.6 & \textbf{82.0}   \\
    Avg   Unseen & 64.1 & 62.9 & 69.9  & 69.9  & 70.6 & 78.1 & 79.1 & 70.3 & 79.8 & \textbf{80.4}   \\
    Avg   All    & 71.3 & 71.8 & 72.7  & 73.8  & 76.6 & 79.5 & 80.1 & 76.3 & 80.9 & \textbf{81.4}   \\ \hline
    Avg   Rank   & 8.8  & 8.2 & 8.0 & 7.8 & 5.5 & 4.3 & 3.2 & 5.8 & 1.9 & \textbf{1.0}   \\       \hline
\end{tabular}
\end{table*}
\subsection{Task-Specific Preconditioner}
\label{subsec:tsp}
Given a test task $\mathcal{T}=\{\mathcal{S}_{\mathcal{T}}, \mathcal{Q}_{\mathcal{T}}\}$, we define Task-Specific Preconditioner $\mathbf{P}^l_{\mathcal{T}}$ as follows:
\begin{equation}
\label{eqn:Definition of TSP}
    \mathbf{P}^l_{\mathcal{T}} = \sum^K_{k=1}p_{\mathcal{T},k} \cdot \mathbf{P}^l_k,\,\,\,\,l=1,\cdots,L
\end{equation}
where $\mathbf{P}^l_k$ is the $l$-th DSP of domain $k$, and $p_{\mathcal{T},k}$ is the task-coefficient for the given task $\mathcal{T}$ and domain $k$. 
By employing $\mathbf{P}^l_{\mathcal{T}}$ as the preconditioning matrix, we can define Task-Specific Preconditioned gradient descent (TSP), as follows:
\begin{equation}
\label{eqn: TSP update}
    \theta^l_{\mathcal{T}, T}=\theta^l_{\mathcal{T},0} - \beta\cdot\sum^{T-1}_{t=0}\mathbf{P}^l_{\mathcal{T}}\nabla_{\theta^l_{\mathcal{T},t}}\mathcal{L}_{\text{in}}(\theta_{\mathcal{T},t};\mathcal{S}_{\mathcal{T}}),
\end{equation}
where $\beta$ is the learning rate used to adapt the task-specific parameters.

\subsection{Positive Definiteness of TSP's Preconditioner}
\label{subsec:pd of tsp}
A preconditioner satisfying positive definiteness ensures a valid Riemannian metric, which represents the geometric characteristics of the parameter space~\cite{amari1967theory, amari1996neural, amari1998natural1,kakade2001natural,amari1998natural2}. 
% Designing a preconditioner as a positive definite matrix ensures a valid Riemannian manifold metric~\cite{amari1998natural,li2016preconditioned}. 
Task-Specific Preconditioner $\mathbf{P}^l_{\mathcal{T}}$ is designed to be a positive definite matrix, which is verified in Theorem~\ref{theorem: positive definiteness of TSP}. 
\begin{theorem}
    \label{theorem: positive definiteness of TSP}
    Let $ p_k \in [0, 1], k=1,\cdots,K$, be the task-coefficients satisfying $\sum^K_{k=1}p_k=1$. For the Domain-Specific Preconditioners $\mathbf{P}_k \in \mathbb{R}^{m \times m}, k=1,\cdots,K $, Task-Specific Preconditioner $\mathbf{P}$ defined as $\mathbf{P} = \sum^K_{k=1}p_k \cdot \mathbf{P}_k$ is positive definite. 
\end{theorem}
The proof is provided in Appendix C. %~\ref{thm:proof}. 
% \textcolor{red}{This shows that TSP is }
% 
Drawing from prior research~\cite{amari1967theory, amari1996neural, amari1998natural1,kakade2001natural, amari1998natural2}, a preconditioner satisfying positive definiteness promotes gradients to point toward the steepest descent direction while avoiding undesirable paths in the parameter space. As shown in Figure~\ref{subfig:motivation2}, positive definiteness improves CDFSL performance, especially in unseen domains. In Section~\ref{sec:discussion}, we will discuss why this property helps in CDFSL. 
% In CDFSL, this property will help find the steepest direction especially in unseen domains, where only the limited information from support set is available.

%% file: Latex_files/5_experiments.tex
\section{Experiments}
\subsection{Experimental Setup}
\subsubsection{Implementation Details} 
In the experiments, we use Meta-Dataset~\cite{triantafillou2019meta} that is the standard benchmark for evaluating the performance of CDFSL. 
% We assess our method in two different scenarios: multi-domain and single-domain settings.  
To demonstrate the effectiveness of TSP as an adaptation mechanism, we apply it to the state-of-the-art CDFSL methods, TSA~\cite{li2022cross} and TA$^2$-Net~\cite{guo2023task}, which are publicly available as open-source. Following previous studies~\cite{bateni2022enhancing,triantafillou2021learning,li2021universal,li2022cross,guo2023task}, we adopted ResNet-18 as the backbone for the feature extractor. In all experiments, we follow the standard protocol described in \cite{triantafillou2019meta}. For the Dataset Classifier Loss, weighting factor $\lambda$ is set to 0.1, as it performs best compared to other values, as shown in Appendix D.1. %~\ref{sec:weighting_factor}. 
Details of the Meta-Dataset, hyper-parameters, and additional implementation are available in Appendix E. %~\ref{sec:imple_details}. 
\subsubsection{Baselines}
For the baselines, we compare our methods to the state-of-the-art CDFSL methods, including BOHB~\cite{saikia2020optimized}, SUR~\cite{dvornik2020selecting}, URT~\cite{liu2020universal}, Simple-CNAPS~\cite{bateni2020improved}, FLUTE~\cite{triantafillou2021learning}, tri-M~\cite{liu2021multi}, URL~\cite{li2021universal}, TSA~\cite{li2022cross}, TA$^2$-Net~\cite{guo2023task}, ALFA~\cite{baik2023learning}+Proto-MAML, GAP+Proto-MAML~\cite{kang2023meta}, and MOKD~\cite{tian2024mokd}. 
\begin{table*}[t]
\setlength{\tabcolsep}{1.5mm}
\centering
\caption{Performance comparison to state-of-the-art methods in a single-domain setting. Mean accuracy and 95$\%$ confidence interval are reported. The best results are highlighted in \textbf{bold}. TSP\textsuperscript{\dag} denotes TSP applied on TSA. TSP\textsuperscript{\dag\dag} denotes TSP applied on TA$^2$-Net.}
\label{tab:SDL results}
\small
% \scalebox{0.51}{
\begin{tabular}{cccccccc|cc}
    \hline
    Test   Dataset   & \begin{tabular}[c]{@{}c@{}}ALFA+\\Proto-MAML\end{tabular} & BOHB & \begin{tabular}[c]{@{}c@{}}GAP+\\Proto-MAML\end{tabular} & FLUTE & TSA  & TA$^2$-Net & MOKD  & TSP\textsuperscript{\dag}  & TSP\textsuperscript{\dag\dag} \\ \hline
    ImageNet         & 52.8$\pm$1.1 & 51.9$\pm$1.1 & 56.7 & 46.9$\pm$1.1 & 59.5$\pm$1.1 & 59.3$\pm$1.1 & 57.3$\pm$1.1 & 60.1$\pm$1.1 & \textbf{60.6$\pm$1.1}  \\ \hline
    Omniglot         & 61.9$\pm$1.5 & 67.6$\pm$1.2 & 77.6 & 61.6$\pm$1.4 & 78.2$\pm$1.2 & 81.1$\pm$1.1 & 70.9$\pm$1.3 & 83.3$\pm$1.1 & \textbf{85.2$\pm$1.1} \\
    Aircraft         & 63.4$\pm$1.1 & 54.1$\pm$0.9 & 68.5 & 48.5$\pm$1.0 & 72.2$\pm$1.0 & 72.6$\pm$0.9 & 59.8$\pm$1.0 & 73.2$\pm$1.0 & \textbf{73.5$\pm$1.1} \\
    Birds            & 69.8$\pm$1.1 & 70.7$\pm$0.9 & 73.5 & 47.9$\pm$1.0 & 74.9$\pm$0.9 & 75.1$\pm$0.9 & 73.6$\pm$0.9 & 76.0$\pm$0.9 & \textbf{76.6$\pm$0.9}  \\
    Textures         & 70.8$\pm$0.9 & 68.3$\pm$0.8 & 71.4 & 63.8$\pm$0.8 & 77.3$\pm$0.7 & 76.8$\pm$0.8 & 76.1$\pm$0.7 & 78.2$\pm$0.7 & \textbf{78.3$\pm$0.7}  \\
    Quick   Draw     & 59.2$\pm$1.2 & 50.3$\pm$1.0 & 65.4 & 57.5$\pm$1.0 & 67.6$\pm$0.9 & 68.4$\pm$0.9 & 61.2$\pm$1.0 & 70.8$\pm$0.9 & \textbf{71.5$\pm$0.9} \\
    Fungi            & 41.5$\pm$1.2 & 41.4$\pm$1.1 & 38.6 & 31.8$\pm$1.0 & 44.7$\pm$1.0 & 45.3$\pm$1.0 & \textbf{47.0$\pm$1.1} & 46.6$\pm$1.0 & \textbf{47.0$\pm$1.0} \\
    VGG   Flower     & 86.0$\pm$0.8 & 87.3$\pm$0.6 & 86.8 & 80.1$\pm$0.9 & 90.9$\pm$0.6 & 91.0$\pm$0.6 & 88.5$\pm$0.6 & 91.8$\pm$0.5 & \textbf{92.2$\pm$0.6}  \\
    Traffic   Sign   & 60.8$\pm$1.3 & 51.8$\pm$1.0 & 66.9 & 46.5$\pm$1.1 & 82.5$\pm$0.8 & 84.1$\pm$0.7 & 61.6$\pm$1.1 & 87.5$\pm$0.8 & \textbf{88.7$\pm$0.8}  \\
    MSCOCO           & 48.1$\pm$1.1 & 48.0$\pm$1.0 & 46.8 & 41.4$\pm$1.0 & 59.0$\pm$1.0 & 58.0$\pm$1.0 & 55.3$\pm$1.0 & \textbf{59.4$\pm$1.0} & 58.6$\pm$1.0 \\
    MNIST            & -    & -    & 94.0 & 80.8$\pm$0.8 & 93.9$\pm$0.6 & 94.9$\pm$0.5 & 88.3$\pm$0.7 & 94.5$\pm$0.5 & \textbf{95.3$\pm$0.6} \\
    CIFAR-10         & -    & -    & 74.5 & 65.4$\pm$0.8 & 82.1$\pm$0.7 & 82.0$\pm$0.7 & 72.2$\pm$0.8 & 83.1$\pm$0.5 & \textbf{83.2$\pm$0.7}  \\
    CIFAR-100        & -    & -    & 63.2 & 52.7$\pm$1.1 & 70.7$\pm$0.9 & 70.8$\pm$0.9 & 63.1$\pm$1.0 & 71.2$\pm$0.9 & \textbf{72.8$\pm$0.9} \\ \hline
    Avg  Seen   & 52.8 & 51.9 & 56.7 & 46.9 & 59.5 & 59.3 & 57.3 & 60.1 & \textbf{60.6} \\
    Avg   Unseen & 62.4 & 59.9 & 68.9 & 56.5 & 74.5 & 75.0 & 68.1 & 76.3 & \textbf{76.9} \\
    Avg   All    & 61.4 & 59.1 & 68.0 & 55.8 & 73.3 & 73.8 & 67.3 & 75.0 & \textbf{75.7} \\ \hline
    Avg   Rank   & 7.0 & 7.5 & 6.1 & 8.9 & 3.7 & 3.4 & 5.1 & 2.0 & \textbf{1.2} \\   \hline  
\end{tabular}
% }
\end{table*}
\subsection{Performance Comparison to State-of-The-Art Methods}
Following the experimental setup in \cite{li2022cross}, we first evaluate our method using multi-domain and single-domain feature extractors in Varying-Way Varying-Shot setting~(i.e., Multi-domain and Single-domain setting). Then, we assess our approach with the multi-domain feature extractor in more challenging Varying-Way Five-Shot and Five-Way One-Shot settings. 
% Full results for the multi-domain and single-domain settings, as well as comparison results for the Varying-Way Five-Shot and Five-Way One-Shot settings, are provided in the Appendix~\ref{sec:additional_results}. 
We provide the performance comparison results for Varying-Way Five-Shot and Five-Way One-Shot settings in the Appendix F. %~\ref{sec:additional_results}. 
% Following the experimental setup in \cite{li2022cross}, we first evaluate our method using multi-domain and single-domain feature extractors in Varying-Way Varying-Shot setting. Then, we assess our approach with the multi-domain feature extractor in more challenging Varying-Way Five-Shot and Five-Way One-Shot settings. 

\subsubsection{Multi-Domain Setting}

In Table~\ref{tab:MDL results}, we evaluate TSP by applying it to TSA and TA$^2$-Net, both of which employ URL~\cite{liu2021multi} as the multi-domain feature extractor. 
We report average accuracies over seen, unseen, and all domains, along with average rank following the previous works~\cite{liu2021multi,li2022cross, guo2023task}. TSP\textsuperscript{\dag} denotes TSP applied on TSA, while TSP\textsuperscript{\dag\dag} indicates TSP applied on TA$^2$-Net. TSP\textsuperscript{\dag} outperforms the previous state-of-the-art methods on 11 out of 13 datasets, and TSP\textsuperscript{\dag\dag} achieves the best results on all datasets. 
For example, TSP\textsuperscript{\dag\dag} outperforms the state-of-the-art method~(TA$^2$-Net) by 1.7\%, 3.4\%, 2.0\%, and 1.9\% on Textures, Fungi, Traffic Sign, and MSCOCO respectively. 
% TSP\textsuperscript{\dag} outperforms state-of-the-art method~(TA$^2$-Net) by 0.9\%, 0.7\%, and 0.8\% on seen, unseen, and all domains, respectively. 
% TSP\textsuperscript{\dag\dag} achieves even greater results of 82.0\%, 80.4\%, and 81.4\% on seen, unsee, and all domains. 
These results imply that TSP can construct a desirable task-specific optimizer that effectively adapt the task-specific parameters for a given target task. 

% The results indicate that TSP\textsuperscript{\dag} demonstrates superior performance, outperforming the previous best results on 11 out of 13 datasets. 
% Additionally, TSP\textsuperscript{\dag\dag} achieves the best results across all datasets. 
% Notably, TSP\textsuperscript{\dag\dag} excels in all 5 unseen datasets, with the average accuracy for unseen datasets notably surpassing the previous best result ($+1.1\%$). 

\subsubsection{Single-Domain Setting}
We evaluate TSP by applying it to TSA and TA$^2$-Net, both of which employ the single-domain feature extractor pretrained solely on the ImageNet dataset. 
% We also evaluate TSP with the single-domain feature extractor trained solely on the ImageNet dataset, following TSA and TA$^2$-Net. 
In Table~\ref{tab:SDL results}, TSP\textsuperscript{\dag\dag} achieves the best results for 12 out of 13 datasets, while TSP\textsuperscript{\dag} leads in the remaining 1 datasets. Compared to recently proposed meta-learning methods based on PGD, such as Approximate GAP$+$Proto-MAML and GAP$+$Proto-MAML~\cite{kang2023meta}, both TSP\textsuperscript{\dag} and TSP\textsuperscript{\dag\dag} consistently outperform them across all 13 datasets by a significant margin. Furthermore, TSP\textsuperscript{\dag\dag} outperforms the previous best methods by a clear margin in several datasets such as Quick Draw~($+3.1\%$), Omniglot~($+4.1\%$), and Traffic Sign~($+4.6\%$). 
Despite being trained only on single dataset, TSP improves performance by effectively constructing a task-specific optimizer tailored to the target task.

\subsection{Ablation Studies}
\label{sec:ablation}
In this section, all ablation studies are performed using TSP\textsuperscript{\dag} in the multi-domain setting to isolate the effects originating from the RL model in TSP\textsuperscript{\dag\dag}. Additional ablation studies are provided in Appendix D. %~\ref{sec:additional_ablation_studies}. 
% In this section, all experiments are conducted on Meta-Dataset with the multi-domain feature extractor~(URL) in a Varying-Way Varying-Shot setting. 
% We employ TSP applied on TSA. 

\subsubsection{Matrix Design for DSP}
\begin{table}
\setlength{\tabcolsep}{2.0mm}
\centering
\caption{Performance comparison of three TSPs with different DSP designs.}
\label{tab:Ablation-pd preconditioners}
\small
\begin{tabular}{ccccc}
    \hline
    DSP designs  & $\mathbf{LL^T}$\;\; & $\mathbf{LL^T+I}$ &\;\;$\mathbf{M^TM+I}$  \\ \hline
    Avg Seen   & 80.8 & 81.2 & \textbf{81.6} \\
    Avg Unseen & 79.0 & 79.4 & \textbf{79.8} \\
    Avg All    & 80.1 & 80.5 & \textbf{80.9} \\ \hline
\end{tabular}
\end{table}
To design Domain-Specific Preconditioner (DSP), we consider three matrix designs that guarantee positive definiteness. The first one is the product of a real-valued lower triangular matrix and its transpose~(i.e., $\mathbf{LL^T}$), where the lower triangular matrix $\mathbf{L}$ is constrained to have positive diagonals. This form is commonly known as the Cholesky factorization~\cite{horn2012matrix}. The second one is the addition of $\mathbf{LL^T}$ and the identity matrix~(i.e., $\mathbf{LL^T+I}$). 
The last one is the addition of the Gram matrix~\cite{horn2012matrix} and the identity matrix ~(i.e., $\mathbf{M^T}\mathbf{M}+\mathbf{I}$). In Table~\ref{tab:Ablation-pd preconditioners}, we compare three TSPs with these three DSP designs. Among them, the Gram matrix design achieves the highest average accuracies in both seen and unseen domains compared to the others. Therefore, we choose the Gram matrix design for DSP. 

\begin{table}
\setlength{\tabcolsep}{2.0mm}
\centering
\caption{Performance comparison of TSPs with and without a PD constraint. $\alpha$ is set to 0.1.}
\small
\label{tab:necessity of PD constraint}
% \scalebox{0.7}{
    \begin{tabular}{ccc}
    \hline
    Preconditioner   & w/ PD constraint & w/o PD constraint \\ \hline
    Avg   Seen   & \textbf{81.6}             & 80.0              \\
    Avg   Unseen & \textbf{79.8}             & 73.8              \\
    Avg   All    & \textbf{80.9}             & 77.6              \\ \hline
\end{tabular}
% }
\end{table}
\begin{table*}
\setlength{\tabcolsep}{2.0mm}
\small
\centering
\caption{The rates of non-PD Domain-Specific Preconditioners (DSPs) after meta-training without a positive definite constraint. For the ResNet-18 backbone, there are 17 DSP preconditioners for each domain. All DSPs are initialized as $0.1\cdot\mathbf{I}$. The average rate is provided in the right column.}
\label{tab:Non-PD rate w/o PD constraint}
% \scalebox{0.75}{
\begin{tabular}{ccccccccc|c}
    \hline
    DSP         & ImageNet & Omniglot & Aircraft & Birds & Textures & Quick Draw & Fungi & VGG Flower & Average \\ \hline
    Non-PD rate & 0.24 & 0.29 & 0.35 & 0.24 & 0.35 & 0.35 & 0.18 & 0.29 & 0.29 \\ 
    \hline
\end{tabular}
% }
\end{table*}
% \subsubsection{Weighting Factor $\lambda$ of the Dataset Classifier Loss} 
% In Table~\ref{tab:Ablation-training loss for dataset classifier}, we compare different dataset classifier losses by adjusting the weighting factor $\lambda$ in Eq.~(\ref{eqn:loss_for_dataset_classifier}). Additionally, we include the results obtained when utilizing the auxiliary loss in Eq.~(\ref{eqn:loss_for_dataset_classifier}) as the dataset classifier loss. From the results, we can observe that $\lambda\:\!=\:\!0.1$ yields the optimal performance, while using other losses results in inferior performance in both seen and unseen domains. This performance gap is more pronounced in unseen domains, highlighting the importance of balancing between the two losses for generalization to unseen domains. Based on these findings, we adopt $\lambda\:\!=\:\!0.1$ for the dataset classifier loss in all experiments presented in this manuscript.

%% file: Latex_files/6_discussion.tex
\section{Discussion}
In this section, all experiments are conducted using TSP\textsuperscript{\dag}.
\label{sec:discussion}
% In this section, we use TSP applied on TSA.
% \paragraph{Examining the Necessity of Positive Definite Constraint}
\paragraph{The Necessity of Positive Definite Constraint}
% Even without a specific constraint of PD, one might assume that initializing the preconditioner as positive definite, such as $\alpha\cdot\mathbf{I}$, would maintain its positive definiteness throughout optimization due to its significant role. 
Even without a specific constraint of PD, one might assume that initializing the preconditioner as positive definite, such as $\alpha\cdot\mathbf{I}$, would maintain its positive definiteness throughout meta-training due to its significant role. 
However, as illustrated in Table~\ref{tab:necessity of PD constraint} and Table~\ref{tab:Non-PD rate w/o PD constraint}, this assumption does not hold. 
In Table~\ref{tab:necessity of PD constraint}, we compare preconditioners with and without a PD constraint, both initialized as positive definite. 
% Specifically, the former adopts TSP~(See Eq.~\ref{eqn: TSP update}), while the latter employs Task-Specific Preconditioner proposed in Eq.~\ref{eqn:Definition of TSP} with DSP designed as $\mathbf{P}^l_k=\mathbf{M}^l_{k}$ and initialized as $\mathbf{M}^l_{k}=0.1\cdot\mathbf{I}$. 
Specifically, the former adopts Task-Specific Preconditioner~(See Eq.~(\ref{eqn:Definition of TSP})), while the latter employs Task-Specific Preconditioner with DSP designed as $\mathbf{P}^l_k=\mathbf{M}^l_{k}$ and initialized as $\mathbf{M}^l_{k}=0.1\cdot\mathbf{I}$. Evaluations are conducted using the multi-domain feature extractor~(URL) in the multi-domain setting. 
% Evaluations are conducted on Meta-Dataset with the multi-domain feature extractor~(URL) in a Varying-Way Varying-Shot setting. 
After meta-training, DSPs without a PD constraint tend to lose positive definiteness as shown in Table~\ref{tab:Non-PD rate w/o PD constraint}, leading to poor performance as shown in Table~\ref{tab:necessity of PD constraint}. These findings underscore the necessity of explicitly constraining the preconditioner to maintain positive definiteness, as relying solely on optimization fails to preserve this crucial property. 

% \paragraph{Comparative Analysis: Positive Definite DSP Designs with and without the Identity Matrix}
\paragraph{Positive Definite DSP Designs with and without the Identity Matrix}
\begin{table}
\small

\centering
\caption{Performance comparison of two positive definite DSP designs with and without adding an identity matrix.}
\label{tab:precondition with identity matrix}
% \scalebox{0.65}{
\setlength{\tabcolsep}{2.0mm}
\begin{tabular}{ccc|cc}
    \hline
                   Setting & \multicolumn{2}{c}{
                   \begin{tabular}[c]
                   {@{}c@{}}Varying-Way\\ Varying-Shot\end{tabular}
                   } & \multicolumn{2}{|c}{
                   \begin{tabular}[c]{@{}c@{}}Varying-Way\\ Five-Shot\end{tabular}
                   }\\ \hline
                   DSP designs   & $\mathbf{L} \mathbf{L}^\mathbf{T}$ & $\mathbf{L} \mathbf{L}^\mathbf{T}+\mathbf{I}$ & $\mathbf{L} \mathbf{L}^\mathbf{T}$ & $\mathbf{L} \mathbf{L}^\mathbf{T}+\mathbf{I}$ \\ \hline
    Avg Seen   & 80.8 & \textbf{81.2} & \textbf{76.8} & 76.6\\
    Avg Unseen & 79.0 & \textbf{79.4} & \textbf{72.1} & 71.5\\
    Avg All    & 80.1 & \textbf{80.5} & \textbf{75.0} & 74.6\\ \hline
\end{tabular}
% }
\end{table}
% (WR comment) Trade-off through identity matrix
% \subsection{The role of an identity matrix in TSP}
Apart from ensuring positive definiteness, a notable characteristic of our Gram matrix design $\mathbf{M^T}\mathbf{M}+\mathbf{I}$ is its inclusion of the identity matrix. To explore the impact of this inclusion, we compare two positive definite DSP designs: $\mathbf{LL^T}$ and $\mathbf{LL^T+I}$. 
% The former, proposed in \cref{sec:ablation}, lacks the identity matrix, while the latter includes it. 
We focus on these two DSP designs because $\mathbf{M^TM}$ does not guarantee positive definiteness. However, we also provide a comparison between $\mathbf{M^TM+I}$ and $\mathbf{M^TM}$ in Appendix G. %~\ref{sec:Additional_results_for_matrix}. 
% \ref{sec:Additional_results_for_matrix}. 
% Although we primarily analyze these two designs because $\mathbf{M^TM}$ does not ensure positive definiteness in this section, we also provide additional comparison results for the other two designs, $\mathbf{M^TM+I}$ and $\mathbf{M^TM}$. in Appendix. 
% We avoid comparing with $\mathbf{M^T}\mathbf{M}$ as it only ensures positive semi-definiteness.
The experiments are conducted using the multi-domain feature extractor~(URL). 
In Table~\ref{tab:precondition with identity matrix}, we observe that the DSP design with the added identity matrix performs better in the Varying-Way Varying-Shot setting but worse in the Varying-Way Five-Shot setting. 
This outcome aligns with prior theoretical findings~\cite{amari2020does} indicating that PGD performs better than GD in noisy gradient conditions, while GD excels when gradients are accurate. With more shots, gradients tend to be more accurate due to increased data. In the Varying-Way Varying Shot setting, where tasks typically involve more than five shots, gradients are more accurate, making GD more beneficial compared to the other setting. Including the identity matrix can be viewed as a regularization of PGD towards GD. Consequently, $\mathbf{L} \mathbf{L}^\mathbf{T} + \mathbf{I}$ aligns closer to GD compared to $\mathbf{L} \mathbf{L}^\mathbf{T}$, resulting in improved performance due to the abundance of shots in Varying-Way Varying-Shot setting. Conversely, in the Varying-Way Five-Shot setting, where tasks involve fewer shots, $\mathbf{L} \mathbf{L}^\mathbf{T}$ exhibits superior performance to $\mathbf{L} \mathbf{L}^\mathbf{T} + \mathbf{I}$ due to the scarcity of shots.

\begin{figure}
    \centering
    \begin{subfigure}[t]{0.2\textwidth}
        \includegraphics[width=\textwidth]{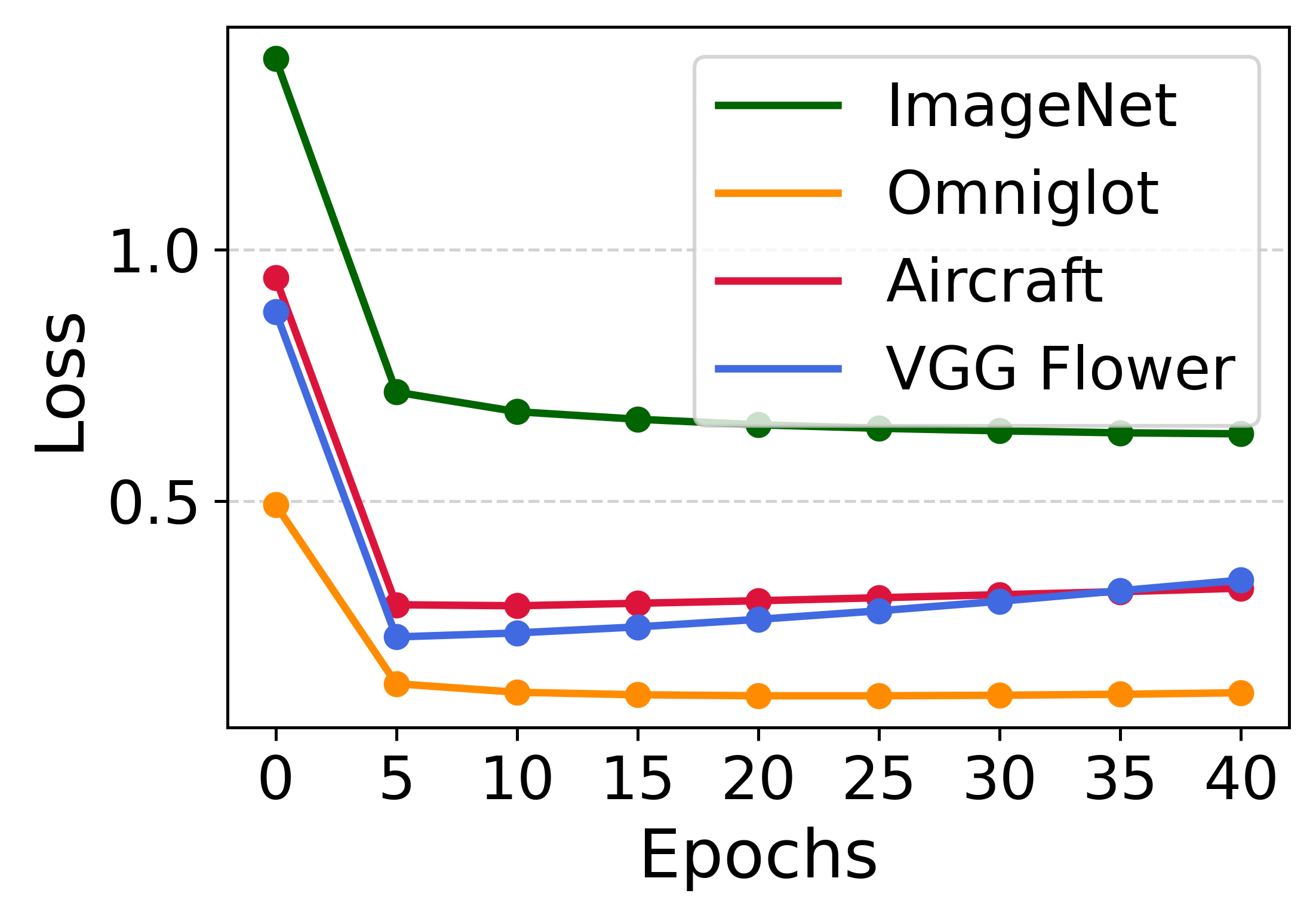}
        \subcaption[]{w/o PD, Seen}
        \label{subfig:npd_seen}
    \end{subfigure}
    \begin{subfigure}[t]{0.2\textwidth}
        \includegraphics[width=\textwidth]{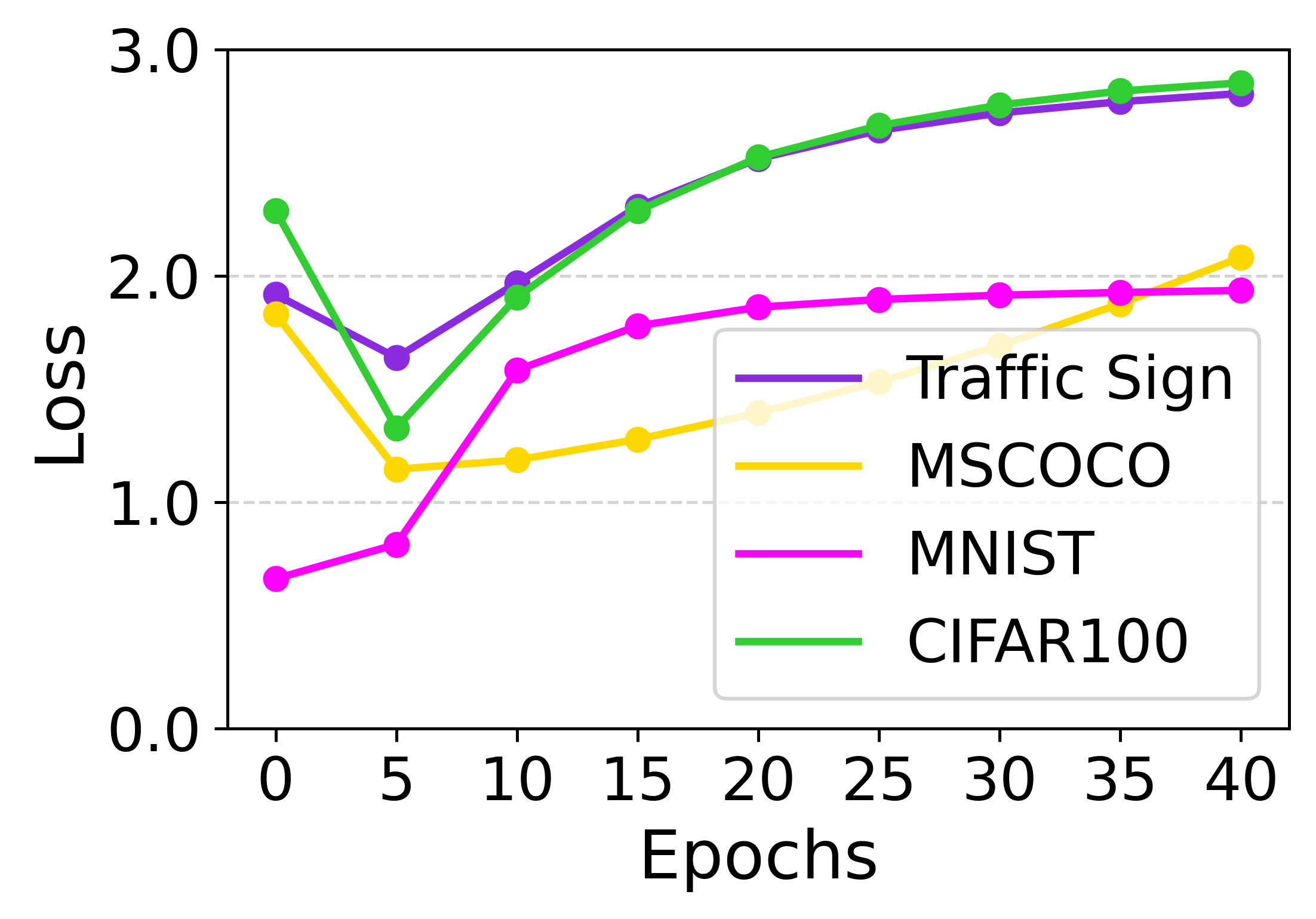}
        \subcaption[]{w/o PD, Unseen}
        \label{subfig:npd_unseen}
    \end{subfigure}
    \begin{subfigure}[t]{0.2\textwidth}
        \includegraphics[width=\textwidth]{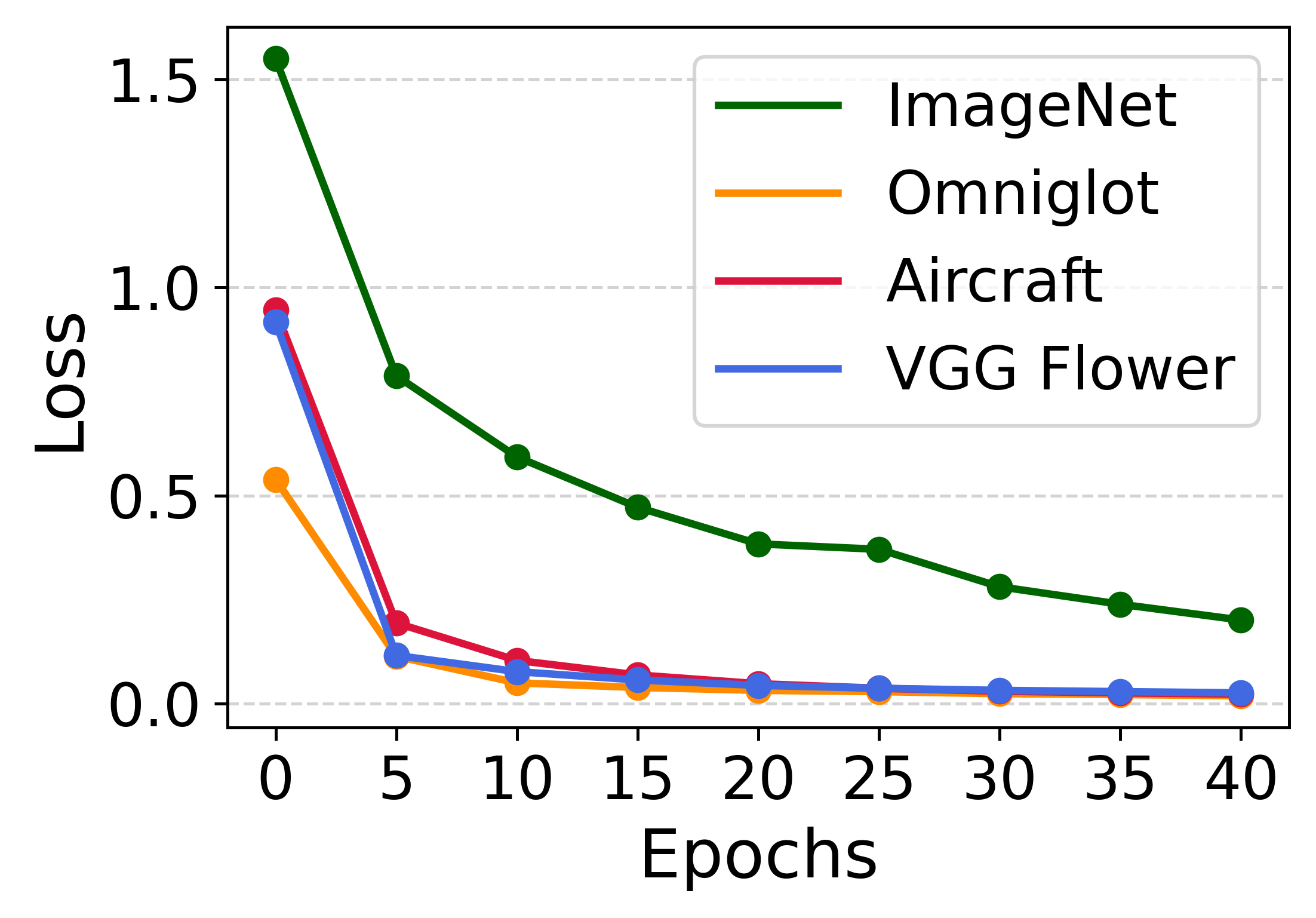}
        \subcaption[]{w/ PD, Seen}
        \label{subfig:pd_seen}
    \end{subfigure}
    \begin{subfigure}[t]{0.2\textwidth}
        \includegraphics[width=\textwidth]{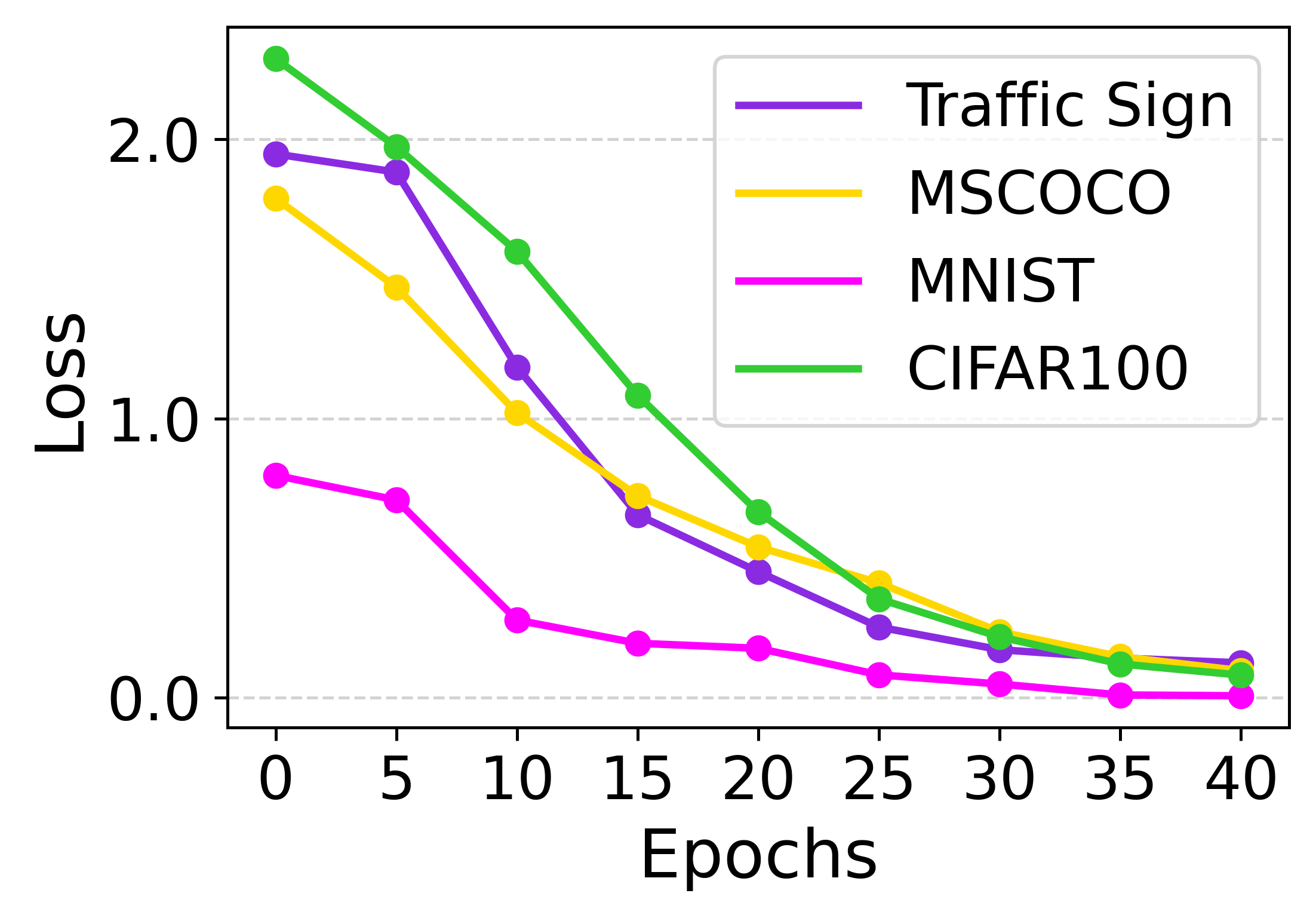}
        \subcaption[]{w/ PD, Unseen}
        \label{subfig:pd_unseen}
    \end{subfigure}
    \caption{Learning curves of PGD with and without the PD constraint across both seen and unseen domains. Further details on the preconditioners used in this figure can be found in Appendix A. %~\ref{sec:three_type}.
    }
    \label{fig:learning_curve}
    % \label{fig:discussion}
\end{figure}

\paragraph{Effectiveness of Positive Definiteness in Cross-Domain Tasks}
A positive definite preconditioner is known to mitigate the negative effects of pathological loss curvature and accelerate optimization, thereby facilitating convergence~\cite{nocedal1999numerical, saad2003iterative, li2017preconditioned}. This leads to a consistent reduction in the objective function. However, without positive definiteness, this effect is not guaranteed and may result in failure to converge. In Figure~4, we compare the learning curves of PGD with and without a PD constraint across both seen and unseen domains. Without the PD constraint, PGD fails to converge in some of the seen domains and in all the unseen domains. With the PD constraint, PGD successfully converges in all the seen and unseen domains. These results suggest that, in cross-domain tasks, a PD constraint of a preconditioner is crucial for achieving convergence and is beneficial for improving performance, which is also related to Figure~\ref{subfig:motivation2}. 

% \paragraph{TSP vs. Previous PGD Methods: A Domain-Specific Approach to CDFSL Challenges} -> ":" 뒤의 문구가 다소 두루뭉실하게 이야기하는 것 같아, 아래에 좀 더 직접적인 문구로 적었습니다. CF. task-specific preconditioner은 우리 논문의 등장하는 '그' task-specific preconditioner를 의미함.
\paragraph{TSP vs. Previous PGD Methods: Leveraging Multi-Domain Knowledge for Task-Specific Preconditioner}
% \label{sec: TSP vs. previsous PGD methods}
% Candidate 1: TSP vs. Previous PGD Methods: A Domain-Specific Approach to CDFSL Challenges
% Cadndiate 2: Beyond Single Preconditioners: TSP’s Superiority Over GAP in Cross-Domain Few-Shot Learning
% Candidate 3: Beyond Single Preconditioners: TSP’s Advantage Over Previous PGD Methods in CDFSL
% Candidate 4: Breaking the Barrier of Single Preconditioners: TSP’s Solution for CDFSL Challenges
% Candidate 5: TSP vs. Previous PGD Methods: Addressing the Limitations of Single Preconditioning
Compared to previous PGD methods like GAP~\cite{kang2023meta}, TSP is specifically designed for cross-domain few-shot learning~(CDFSL), where unseen domains are not accessed during meta-training. The key challenge in CDFSL is to effectively leverage information from multiple seen domains to quickly adapt to each unseen domain. Previous PGD methods fall short in this regard because they rely on a single preconditioner, even when multiple seen domains are available. For example, GAP uses only one preconditioner to extract information from multiple seen domains, which limits its adaptability to unseen domains with distinct characteristics. In contrast, TSP meta-trains a distinct domain-specific preconditioner~(DSP) for each seen domain and combines them to construct a Task-Specific Preconditioner that better suited to each unseen domain. TSP produces this Task-Specific Preconditioner \emph{effectively}, as shown in Tables~\ref{tab:MDL results} and \ref{tab:SDL results}, and \emph{time-efficiently}, as further detailed in Appendix H. %~\ref{sec:time-efficiency of TSP compared to GAP}.

%% file: Latex_files/7_conclusion.tex
\section{Conclusion}
In this study, we have introduced a robust and effective adaptation mechanism called Task-Specific Preconditioned gradient descent~(TSP) to enhance CDFSL performance. Thanks to the meta-trained Domain-Specific Preconditioners~(DSPs) and Task-coefficients, TSP can flexibly adjust the optimization strategy according to the geometric characteristics of the parameter space for the target task. Owing to these components, the proposed TSP demonstrates notable performance improvements on Meta-Dataset across various settings. 

%% file: Latex_files/8_acknowledgement.tex
\section{Acknowledgements}
% Full version
This work was supported by a National Research Foundation of Korea (NRF) grant funded by the Korea government (MSIT) (No. NRF-2020R1A2C2007139) and in part by Institute of Information \& communications Technology Planning \& Evaluation (IITP) grant funded by the Korea government (MSIT) ([NO.RS-2021-II211343, Artificial Intelligence Graduate School Program (Seoul National University)], [No. RS-2023-00235293, Development of autonomous driving big data processing, management, search, and sharing interface technology to provide autonomous driving data according to the purpose of usage]).

% Shorten version
% This work was supported by NRF(No. NRF-2020R1A2C2007139) and IITP([No.RS-2021-II211343, Artificial Intelligence Graduate School Program (Seoul National University)], [No. RS-2023-00235293, Development of autonomous driving big data processing, management, search, and sharing interface technology to provide autonomous driving data according to the purpose of usage])

%% file: Latex_files/9_appendix.tex
\appendix
\twocolumn[{%
 \centering
 \Large \textbf{Appendix for the paper \\ ``Task-Specific Preconditioner for Cross-Domain Few-Shot Learning''}\\[1em]
}]

\setcounter{theorem}{0}
\setcounter{lemma}{0}

\section{The three preconditioners used in Figure~\ref{subfig:motivation2} and Figure~\ref{fig:learning_curve}}
\label{sec:three_type}
To establish the motivation for enforcing the positive definite constraint in CDFSL, we conduct a comparative analysis of three adaptation mechanisms---PGD methods with varying preconditioners---using Meta-Dataset. These mechanisms are applied on the state-of-the-art CDFSL method, TSA~\cite{li2022cross}. 
In these comparisons, with task-specific parameters $\theta$ and a task $\mathcal{T}=\{\mathcal{S}_{\mathcal{T}},\mathcal{Q}_{\mathcal{T}}\}$, we update $\theta$ using PGD with a preconditioner $\mathbf{P}$ as follows:
\begin{equation}
    \theta_{\mathcal{T},t}=\theta_{\mathcal{T}, t-1}-\alpha\cdot\mathbf{P}\nabla_{\theta}\mathcal{L}(\theta_{\mathcal{T}, t-1};\mathcal{S}_{\mathcal{T}}),\;t=1,2,\cdots,
\end{equation}
where $\theta_{\mathcal{T},0}=\theta$ and $\mathcal{L}(\theta_{\mathcal{T},t};\mathcal{S}_{\mathcal{T}})$ is the empirical loss associated with $\mathcal{T}$ and $\theta_{\mathcal{T},t}$. 
The first PGD method is identical to Gradient Descent~(GD), which utilizes the fixed identity matrix $\mathbf{I}$ as $\mathbf{P}$~(i.e., the baseline for gradient descent). 
The second method is Task-Specific Preconditioned Gradient Descent (TSP), which utilizes a Task-Specific Preconditioner designed as $\mathbf{P}^l_k = \mathbf{M}^l_k$ and initialized as $(-1) \cdot \mathbf{I}$~(i.e., PGD without a positive definite constraint). 
% This preconditioner does not enforce a positive definite constraint. 
% which employ a matrix that is defined by each domains and does not enforce a positive definite constraint as follows:
% \begin{equation}
% \end{equation}
% The second method does not enforce a positive definite constraint and employs an adaptable preconditioner initialized as $(-1)\cdot\mathbf{I}$. 
The final method is Task-Specific Preconditioned gradient descent~(TSP), which utilizes Task-Specific Preconditioner defined in Section~\ref{subsec:tsp}~(i.e., PGD with a positive definite). 

\section{Meta-Training and Meta-Testing Algorithms}
\label{sec:pseudocode}
% \subsection{Training and Testing Algorithms}
\begin{algorithm}
    \caption{Meta-Training for Domain-Specific Preconditioner (DSP)}\label{tsp:alg_dsp}
    \begin{algorithmic}[1]
        \REQUIRE $p(\mathcal{T})$: Task distribution across $K$ train domains
        \REQUIRE $\alpha_{\text{in}}$, $\alpha_{\text{out}}$: The learning rates
        \REQUIRE $\mathcal{L}_{\text{in}}$, $\mathcal{L}_{\text{out}}$: The inner and outer-level loss functions
        \STATE Initialize task-specific parameters \begin{equation*}
            \theta=\{\theta^l\in\mathbb{R}^{m_l\times m_l}\}^L_{l=1}
        \end{equation*}
        \STATE Initialize meta-parameters \begin{equation*}
            \mathcal{M}_1,\cdots,\mathcal{M}_K \text{ where } \mathcal{M}_k=\{\mathbf{M}^l_k\in\mathbb{R}^{m_l\times m_l}\}^L_{l=1}
        \end{equation*}
        \WHILE{not converged}
            \STATE Sample a batch of train tasks $\mathcal{T}_B \sim p(\mathcal{T})$
            \FOR{all $\mathcal{T}=\{\mathcal{S}_{\mathcal{T}},\mathcal{Q}_{\mathcal{T}},d_{\mathcal{T}}\}\in\mathcal{T}_B$}
                \FOR{$l=1$ \textbf{to} $L$}
                    \STATE Compute DSP $\mathbf{P}^l_{d_{\mathcal{T}}}=\mathbf{M}_{d_{\mathcal{T}}}^{l\mathbf{T}}\mathbf{M}^l_{d_{\mathcal{T}}}+\mathbf{I}$
                    \STATE Compute updated task-specific parameters via Eq.~(\ref{eqn:definition_inner_opt})
                \ENDFOR
                \STATE Compute outer-level loss $\mathcal{L}_{\text{out}}(\theta_{\mathcal{T},T};\mathcal{Q}_{\mathcal{T}})$
            \ENDFOR
            \STATE Update the meta-parameters via Eq.~(\ref{eqn:definition_outer_opt})
        \ENDWHILE
    \end{algorithmic}
\end{algorithm}

\begin{algorithm}
    \caption{Meta-Training the dataset classifier for Task-coefficients}\label{tsp:alg_ts}
    \begin{algorithmic}[1]
        \REQUIRE $p(\mathcal{T})$: Task distribution across $K$ train domains
        \REQUIRE $g_{\phi}(\cdot)$: The dataset classifier
        \REQUIRE $\mathcal{L}_{\text{in}}$, $\mathcal{L}_{\text{out}}$: The inner and outer-level loss functions
        \REQUIRE $\lambda$: The regularization parameter
        \REQUIRE $\alpha$: The learning rate
        \STATE Initialize task-specific parameters \begin{equation*}
            \theta=\{\theta^l\in\mathbb{R}^{m_l\times m_l}\}^L_{l=1}
        \end{equation*}
        \STATE Initialize the parameters $\phi$ of $g_{\phi}(\cdot)$
        \WHILE{not converged}
            \STATE Sample a batch of train tasks $\mathcal{T}_B \sim p(\mathcal{T})$
            \FOR{all $\mathcal{T}=\{\mathcal{S}_{\mathcal{T}},\mathcal{Q}_{\mathcal{T}},d_{\mathcal{T}}\}\in\mathcal{T}_B$}
                \STATE Compute logits $(z_{\tau, 1},\cdots,z_{\tau, K})=g(\mathcal{S}_{\tau})$
                \STATE Compute task-coefficients via Eq.~(\ref{eqn:task_coefficients})
                % \begin{equation*}
                %     (p_{\tau, 1},\cdots,p_{\tau, K})=\text{Softmax}\big(z_{\tau, 1},\cdots,z_{\tau, K}\big)
                % \end{equation*}
                \STATE Compute $\mathcal{L}_{\text{CE}}^{\mathcal{T}}=-\sum^K_{k=1}d_{\mathcal{T},k} \cdot \log(p_{\mathcal{T},k})$
                % For domain label $d_{\mathcal{T}}=(d_{\mathcal{T},1},\cdots,d_{\mathcal{T},K})$, compute
                % \begin{equation*}
                    % \mathcal{L}_{\text{CE}}^{\mathcal{T}}=-\sum^K_{k=1}d_{\mathcal{T},k} \cdot \log(p_{\mathcal{T},k})
                % \end{equation*}
                \STATE With $\theta_{\mathcal{T},T}=\{\theta_{\mathcal{T},T}^l\}^L_{l=1}$ via Eq.~(\ref{eqn:14}), compute 
                \begin{equation*}
                    \mathcal{L}_{\text{Aux}}^{\mathcal{T}}=\mathcal{L}_{\text{out}}(\theta_{\mathcal{T},T};\mathcal{Q}_{\mathcal{T}})
                \end{equation*}
                \STATE Compute $\mathcal{L}^\mathcal{T}=\mathcal{L}_{\text{CE}}^{\mathcal{T}} + \lambda\cdot \mathcal{L}_{\text{Aux}}^{\mathcal{T}}$
            \ENDFOR
            \STATE Compute $\mathcal{L}=\sum_{\mathcal{T}\in\mathcal{T}_B}\mathcal{L}^{\mathcal{T}}$
            \STATE Update the parameters 
            \begin{equation*}
                \phi\leftarrow\phi-\alpha\cdot\nabla_{\phi}\mathcal{L}
            \end{equation*}
            % $\phi$ of the dataset classifier based on $\mathcal{L}$
        \ENDWHILE
    \end{algorithmic}
\end{algorithm}

\begin{algorithm}
    \caption{Meta-Testing through Task-Specific Preconditioned gradient descent~(TSP)}\label{tsp:alg_tsp}
    \begin{algorithmic}[1]
        \REQUIRE $p(\mathcal{T})$: Task distribution across all domains
        \REQUIRE $\mathcal{M}_1,\cdots,\mathcal{M}_K$: Meta-trained meta-parameters
        \REQUIRE $g_{\phi}(\cdot)$: Meta-trained dataset classifier
        \REQUIRE $\mathcal{L}_{\text{in}}$: The inner-level loss function
        \REQUIRE $\beta$: The learning rate
        \STATE Initialize task-specific parameters \begin{equation*}
            \theta=\{\theta^l\in\mathbb{R}^{m_l\times m_l}\}^L_{l=1}
        \end{equation*}
        \STATE Sample a test task $\mathcal{T}=\{\mathcal{S}_{\mathcal{T}},\mathcal{Q}_{\mathcal{T}}\}$
        \FOR{$l=1$ \textbf{to} $L$}
            \STATE For all $k$, compute DSP $\mathbf{P}^l_k=\mathbf{M}_k^{l\mathbf{T}}\mathbf{M}_k^l+\mathbf{I}$
            \STATE Compute Task-Specific Preconditioner via Eq.~(\ref{eqn:Definition of TSP})
            \STATE Update the task-specific parameters via Eq.~(\ref{eqn: TSP update}) 
            % \begin{equation*}
            %     \theta^l_{\tau, T}=\theta^l_{\tau,0} - \beta\cdot\sum^{T-1}_{s=0}\mathbf{P}^l_{\tau}\nabla_{\theta^l_{\tau,s}}\mathcal{L}^{\text{in}}(\theta_{\tau,s};\mathcal{S}_{\tau})
            % \end{equation*}
        \ENDFOR
    \end{algorithmic}
\end{algorithm}

\section{Proofs of Theorems}
\label{thm:proof}
\begin{lemma}
\label{lemma: positive definiteness of DSP}
For the meta parameter $\mathbf{M} \in \mathbb{R}^{m \times m}$, the Domain-Specific Preconditioner $\mathbf{P}$ defined as $\mathbf{P} = \mathbf{M}^{\mathbf{T}} \mathbf{M} + \mathbf{I}$ is positive definite.
\end{lemma}

\begin{proof}
$\mathbf{P}$ is symmetric, as shown below:
\begin{equation*} 
\mathbf{P}^\mathbf{T} = (\mathbf{M}^\mathbf{T} \mathbf{M} + \mathbf{I})^\mathbf{T} = (\mathbf{M}^\mathbf{T} \mathbf{M})^\mathbf{T} + \mathbf{I}^\mathbf{T} = \mathbf{M}^\mathbf{T} \mathbf{M} + \mathbf{I} = \mathbf{P}.
\end{equation*}
For $\forall \, \mathbf{x} \in \mathbb{R}^m \backslash \{\mathbf{0}\}$, 
\begin{align*} 
\mathbf{x}^\mathbf{T} \mathbf{P} \mathbf{x} &= \mathbf{x}^\mathbf{T} (\mathbf{M}^\mathbf{T} \mathbf{M} + \mathbf{I}) \mathbf{x} \\ 
&= \mathbf{x}^\mathbf{T} \mathbf{M}^\mathbf{T} \mathbf{M} \mathbf{x} + \mathbf{x}^\mathbf{T} \mathbf{I} \mathbf{x} \\
&= (\mathbf{M} \mathbf{x})^\mathbf{T} (\mathbf{M} \mathbf{x}) + \mathbf{x}^\mathbf{T} \mathbf{x} \\
&= \| \mathbf{M} \mathbf{x} \|^2 + \| \mathbf{x} \|^2.
\end{align*}
The first term on the right-hand side is non-negative, while the second term is positive:
\begin{align*}
&\| \mathbf{M} \mathbf{x} \|^2 \geq 0, \\
&\| \mathbf{x} \|^2 > 0
\end{align*}
since $\mathbf{x} \neq \mathbf{0}$. Thus, we conclude:
\begin{equation}
\label{eqn:quadratic form of DSP is positive}
\mathbf{x}^\mathbf{T} \mathbf{P} \mathbf{x} > 0,
\end{equation}
which confirms the positive-definiteness of the Domain-Specific Preconditioner $\mathbf{P}$.
\end{proof}

% \subsection{Proof of Theorem 1}
% \label{thm:proof}
\begin{theorem}
    \label{theorem: positive definiteness of TSP - in appendix}
    Let $ p_k \in [0, 1], k=1,\cdots,K$, be the task-coefficients satisfying $\sum^K_{k=1}p_k=1$. For the Domain-Specific Preconditioners $\mathbf{P}_k \in \mathbb{R}^{m \times m}, k=1,\cdots,K $, Task-Specific Preconditioner $\mathbf{P}$ defined as $\mathbf{P} = \sum^K_{k=1}p_k \cdot \mathbf{P}_k$ is positive definite. 
    % For \textcolor{red}{a given} positive definite Domain Specific Preconditioner\textcolor{blue}{s} $\mathbf{P}_k$ and $0 < p_k < 1$ \textcolor{red}{for $k = 1,\cdots,K$}, the Task Specific Preconditioner $\mathbf{P} = \sum^K_{k=1}p_k \cdot \mathbf{P}_k$  is positive definite.
\end{theorem}
\begin{proof}
    By Lemma~\ref{lemma: positive definiteness of DSP}, $\mathbf{P}_k$ is symmetric. Therefore, $\mathbf{P}$ is symmetric, as shown below:
    \begin{equation*} 
        \mathbf{P}^\mathbf{T} = \left(\sum^K_{k=1}p_k \cdot \mathbf{P}_k\right)^\mathbf{T} = \sum^K_{k=1}p_k \cdot \mathbf{P}_k^\mathbf{T} = \sum^K_{k=1}p_k \cdot \mathbf{P}_k = \mathbf{P}.
    \end{equation*}
    For $\forall \, \mathbf{x} \in \mathbb{R}^m \backslash \{\mathbf{0}\}$, 
    \begin{align*} 
        \mathbf{x}^\mathbf{T} \mathbf{P} \mathbf{x}
        &= \mathbf{x}^\mathbf{T} \left(\sum^K_{k=1}p_k \cdot \mathbf{P}_k\right) \mathbf{x} \\
        &= \sum^K_{k=1}p_k \cdot \mathbf{x}^\mathbf{T} \mathbf{P}_k \mathbf{x}.
    \end{align*}
    Since $\mathbf{P}_k$ is the Domain-Specific Preconditioner, by Lemma~\ref{lemma: positive definiteness of DSP}, each summand on the right-hand side is non-negative (see Eq.~(\ref{eqn:quadratic form of DSP is positive})):
    \begin{equation}
    \label{eqn:summand in TSP}
        p_k \cdot \mathbf{x}^\mathbf{T} \mathbf{P}_k \mathbf{x} \geq 0,\;\; k=1,\cdots,K
    \end{equation}
    because $0 \leq p_k \leq 1$. Since $\sum^K_{k=1}p_k = 1$ and $p_k \in [0, 1], k=1,\cdots,K$, there exists at least one $p_k$ such that $p_k > 0$, implying that at least one term in Eq.~(\ref{eqn:summand in TSP}) is positive: 
    \begin{equation}
    \label{eqn:positive summand in TSP}
        \exists \, k \in \{ 1, 2, 3, \dots, K\} \:\; \text{such that} \:\; p_k \cdot \mathbf{x}^\mathbf{T} \mathbf{P}_k \mathbf{x} > 0.
    \end{equation}
    Combining Eq.~(\ref{eqn:summand in TSP}) and Eq.~(\ref{eqn:positive summand in TSP}), we conclude:
    \begin{equation*}
        \mathbf{x}^\mathbf{T} \mathbf{P} \mathbf{x} = \sum^K_{k=1}p_k \cdot \mathbf{x}^\mathbf{T} \mathbf{P}_k \mathbf{x} > 0,
    \end{equation*}
    which confirms the positive-definiteness of the Task-Specific Preconditioner $\mathbf{P}$.
\end{proof}

\section{Additional Ablation Studies}
\label{sec:additional_ablation_studies}
\subsection{Weighting Factor $\lambda$ of the Dataset Classifier Loss} 
\label{sec:weighting_factor}
In Table~\ref{tab:Ablation-training loss for dataset classifier}, we compare different dataset classifier losses by adjusting the weighting factor $\lambda$ in Eq.~(\ref{eqn:loss_for_dataset_classifier}). Additionally, we include the results obtained when utilizing the auxiliary loss in Eq.~(\ref{eqn:loss_for_dataset_classifier}) as the dataset classifier loss. From the results, we can observe that $\lambda\:\!=\:\!0.1$ yields the optimal performance, while using other losses results in inferior performance in both seen and unseen domains. This performance gap is more pronounced in unseen domains, highlighting the importance of balancing between the two losses for generalization to unseen domains. Based on these findings, we adopt $\lambda\:\!=\:\!0.1$ for the dataset classifier loss in all experiments presented in this manuscript.
\begin{table}[t]
\small
\setlength{\tabcolsep}{1.0mm}
\centering
\caption{Mean accuracy ($\%$) for different values of $\lambda$ in the dataset classifier loss (See Eq.~(\ref{eqn:loss_for_dataset_classifier})).}
\label{tab:Ablation-training loss for dataset classifier}
\begin{tabular}{cccccccccc}
    \hline
        & $\lambda=10$ & $\lambda=1$ & \begin{tabular}[c]{@{}c@{}}$\lambda=0.1$\\ (Ours)\end{tabular} & $\lambda=0.01$ & \begin{tabular}[c]{@{}c@{}}Only\\$\mathcal{L}_{\text{CE}}$\end{tabular} & \begin{tabular}[c]{@{}c@{}}Only\\$\mathcal{L}_{\text{Aux}}$\end{tabular}\\ \hline
    Avg Seen   & 80.8 & 81.3 & \textbf{81.6} & 81.0 & 80.9 & 80.1 \\
    Avg Unseen & 79.3 & 79.7 & \textbf{79.8} & 79.1 & 78.8 & 78.2 \\
    Avg All    & 80.2 & 80.7 & \textbf{80.9} & 80.3 & 80.1 & 79.4 \\ \hline
\end{tabular}
\end{table}

\subsection{Interpreting and visualizing task-coefficient}
To better understand how DSPs from the eight training domains combine to form Task-Specific Preconditioner, we illustrate the task-coefficients of TSP used in various test tasks in Figure~\ref{fig:task_coeff}. We first randomly sample the test tasks from each of the four domains (ImageNet, Birds, Traffic Sign, and MSCOCO). 
% The test tasks are randomly sampled from each of the four domains (ImageNet, Birds, Traffic Sign, and MSCOCO). 
The blue heatmaps illustrate the task-coefficient values utilized in each test task. 
These heatmaps exhibit consistent patterns within each domain, although the values vary across tasks. For instance, in the Birds domain, all 5 tasks primarily rely on DSPs from ImageNet and Birds. Meanwhile, Task 2 evenly distributes task-coefficient values between them, while Task 4 assigns significantly more values to the Birds DSP compared to the ImageNet counterpart. 
In Figure~\ref{fig:task_coeff_extra1}, Figure~\ref{fig:task_coeff_extra2}, and Figure~\ref{fig:task_coeff_extra3}, we randomly sample the test tasks from other domains. 
% present additional visualization results for test tasks from other domains. 
% we present additional visualization results illustrating the task coefficients of Task-Specific Preconditioner utilized in various test tasks sampled from the remaining domain not provided in Section~\ref{sec:tc_visual}. 
Similar to the patterns observed in the heatmaps presented in Figure~\ref{fig:task_coeff}, these heatmaps also demonstrate consistent patterns within each domain, although the values vary across tasks. 

\begin{figure*}[t]
    \begin{subfigure}[t]{0.5\textwidth}
        \includegraphics[width=1.0\textwidth]{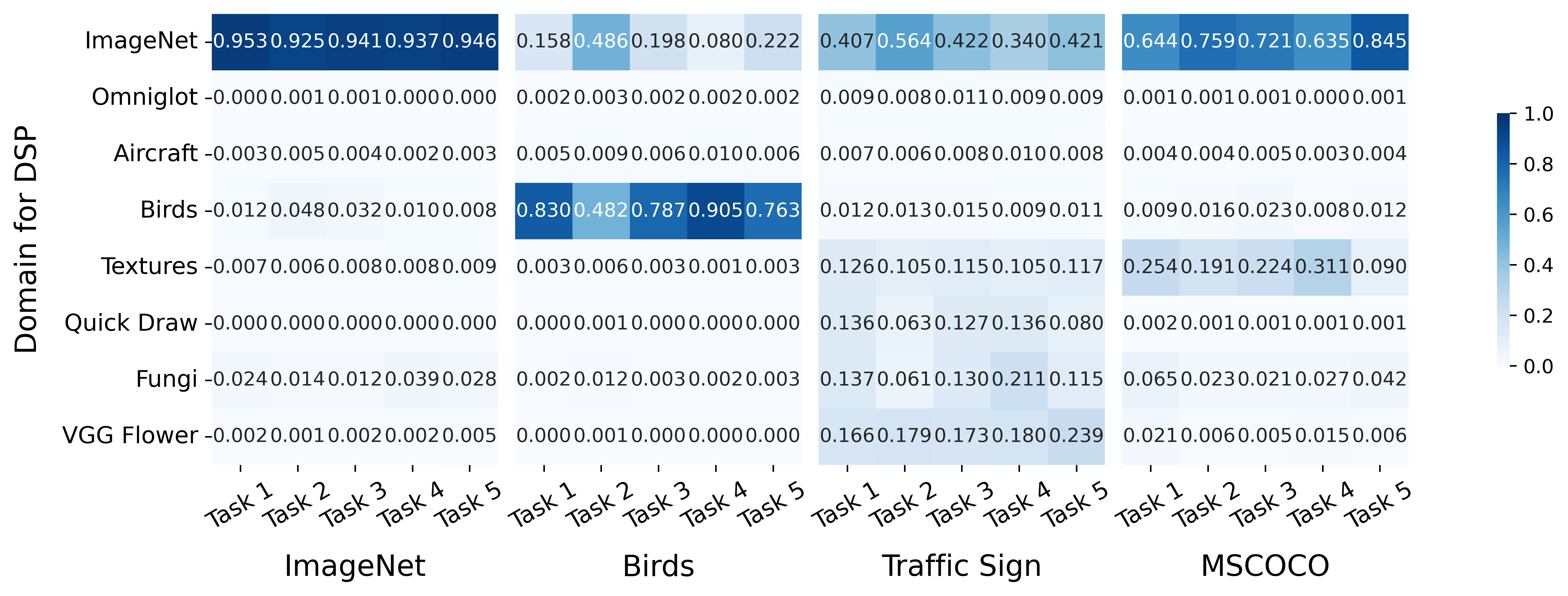}
        \subcaption{}
        \label{fig:task_coeff}
    \end{subfigure}
    \begin{subfigure}[t]{0.5\textwidth}
        \includegraphics[width=1.0\textwidth]{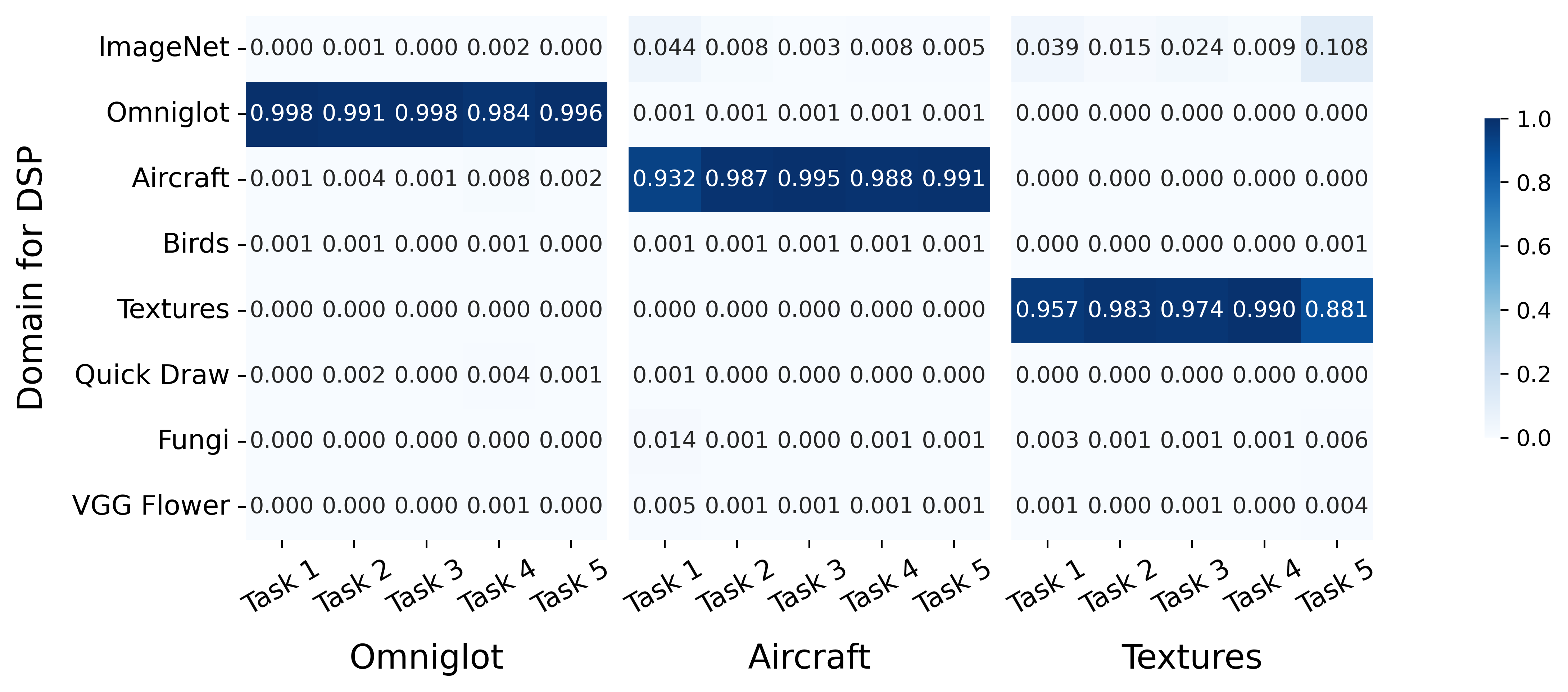}
        \subcaption{}
        \label{fig:task_coeff_extra1}
    \end{subfigure}
    
    \begin{subfigure}[t]{0.5\textwidth}
        \includegraphics[width=1.0\textwidth]{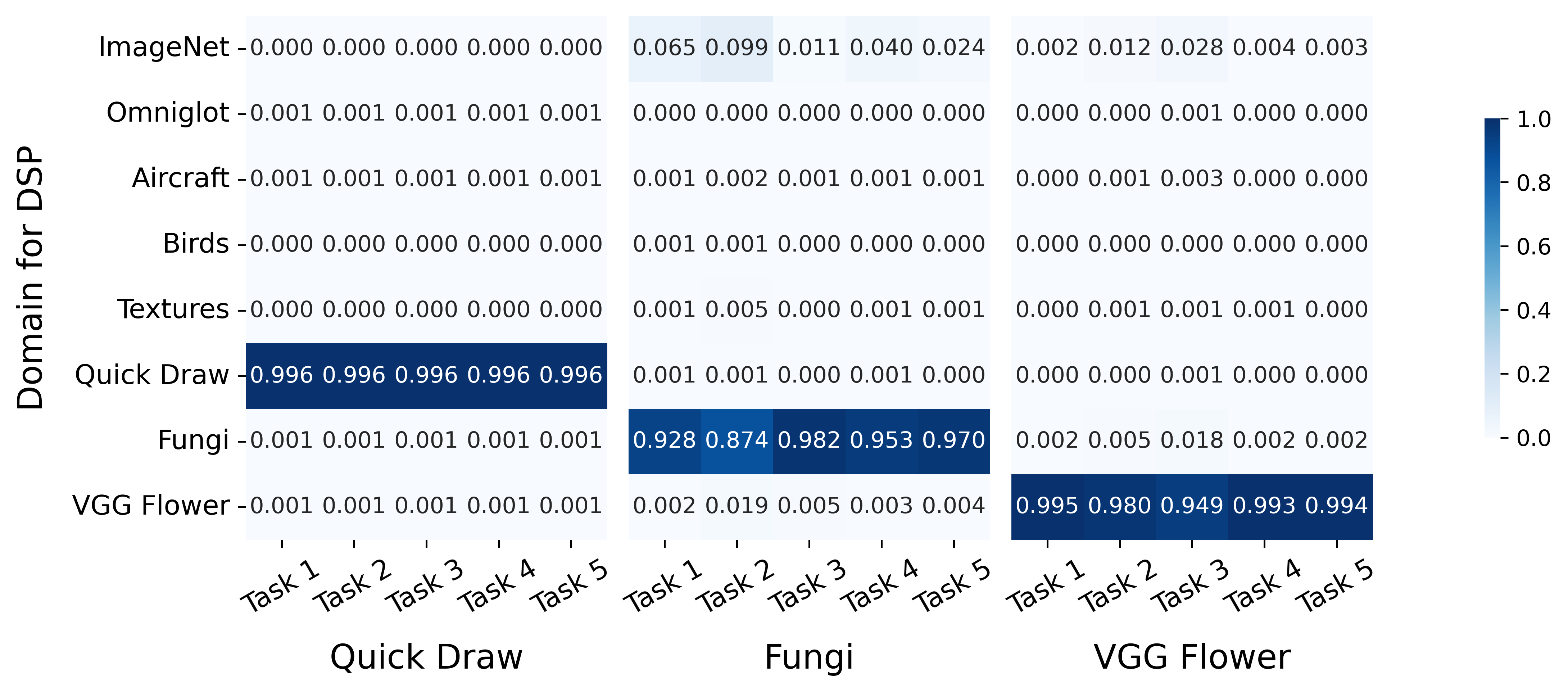}
        \subcaption{}
        \label{fig:task_coeff_extra2}
    \end{subfigure}
    \begin{subfigure}[t]{0.5\textwidth}
        \includegraphics[width=1.0\textwidth]{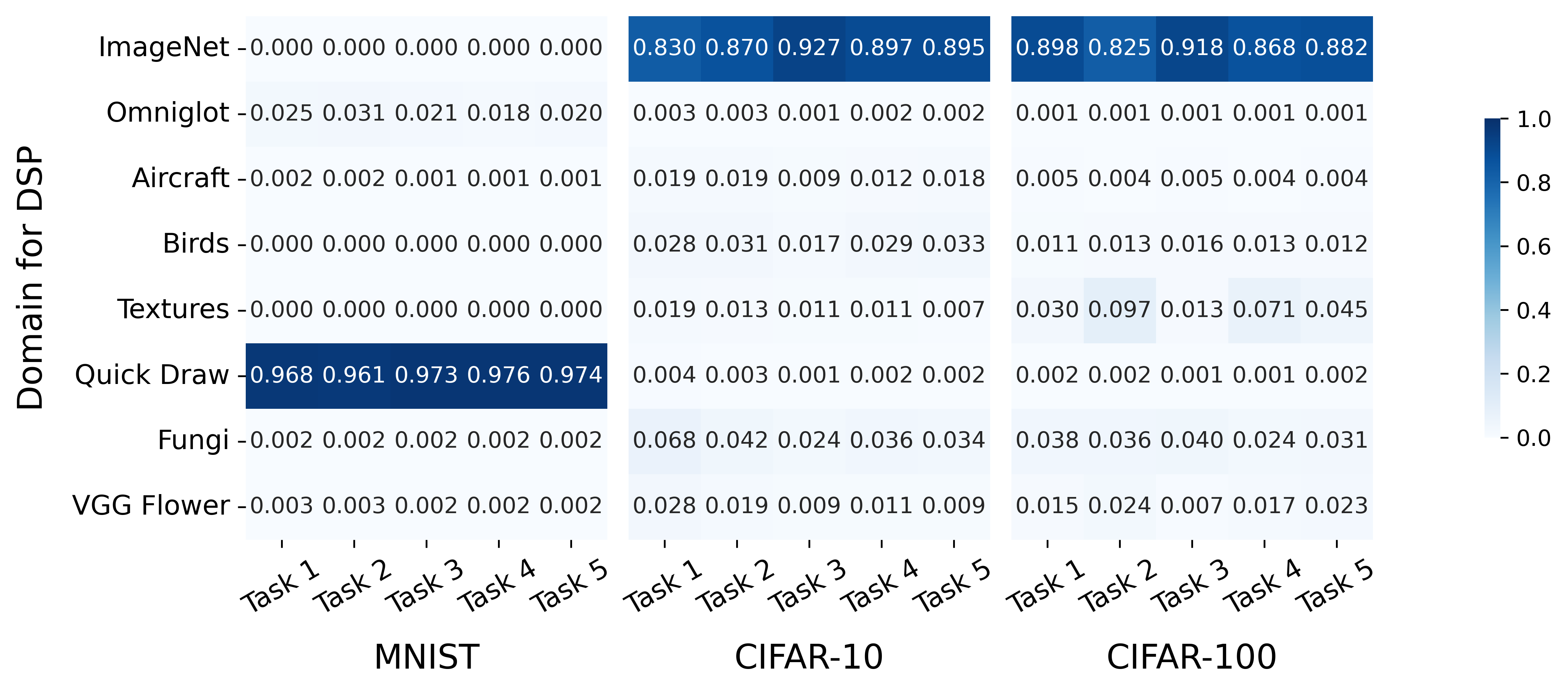}
        \subcaption{}
        \label{fig:task_coeff_extra3}
    \end{subfigure}
    \caption{Task-coefficient values used in the construction of Task-Specific Preconditioner. Columns represent DSPs trained on one of the eight training domains of Meta-Dataset. (a) Rows represent five test tasks randomly sampled from each of the four domains: 2 seen domains (ImageNet, Birds) and 2 unseen domains (Traffic Sign, and MSCOCO). (b) Rows represent five test tasks randomly sampled from each of the three seen domains: Omniglot, Aricraft, and Textures. (c) Rows represent five test tasks randomly sampled from each of the three seen domains: Quick Draw, Fungi, and VGG Flower. (d) Rows represent five test tasks randomly sampled from each of the three unseen domains: MNIST, CIFAR-10, and CIFAR-100.}
    \label{fig:task_coeff_full}
\end{figure*}

\section{Implementation Details}
\label{sec:imple_details}
\subsection{Dataset}
Meta-Dataset~\cite{triantafillou2019meta} is the standard benchmark for evaluating the performance of cross-domain few-shot classification. Initially, it comprised ten datasets, including ILSVRC 2012~\cite{russakovsky2015imagenet}, Omniglot~\cite{lake2015human}, FGVC-Aircraft~\cite{maji2013fine}, CUB-200-2011~\cite{wah2011caltech}, Describable Textures~\cite{cimpoi2014describing}, QuickDraw~\cite{ha2017neural}, FGVCx Fungi~\cite{schroeder2018fgvcx}, VGG Flower~\cite{nilsback2008automated}, Traffic Signs~\cite{houben2013detection}, and MSCOCO~\cite{lin2014microsoft}. Later, it was further expanded to include MNIST~\cite{lecun1998gradient}, CIFAR-10~\cite{krizhevsky2009learning}, and CIFAR-100~\cite{krizhevsky2009learning}. 

\subsection{Architecture for the dataset classifier}
\label{sec:architecture}
As the dataset classifier, we use a permutation-invariant set encoder $g$~\cite{zaheer2017deep} followed by a linear layer. We adopt the implementation of a permutation-invariant set encoder as described in previous studies~\cite{requeima2019fast,triantafillou2021learning}. We implement this encoder as Conv-$5$ backbone, comprising $5$ modules with $3\times3$ convolutions employing $256$ filters, followed by batch normalization, ReLU activation, and $2\times2$ max-pooling with a stride of $2$. Subsequently, global average pooling is applied to the output, followed by averaging over the first dimension (representing different examples within the support set), resulting in the set representation of the given support set. This representation is then fed into a linear layer to classify the given support set into one of the $K$-training datasets. 

\subsection{Hyper-parameters}
\label{sec:hyperparameters}
% \subsubsection{Varying-Way Varying Shot, Varying-Way Five-Shot, and Five-Way One-Shot Setting}
% For Varying-Way Varying Shot, Varying-Way Five-Shot, and Five-Way One-Shot setting, we use 
% use the batch size of 16, 
For all the experiments, we use the hyper-parameters in Table~\ref{tab:dsp_hyperparameters} and Table~\ref{tab:dataset_classifier_hyperparameters}.  

\begin{table*}
\centering
\setlength{\tabcolsep}{0.5mm}
\small
\caption{Hyper-parameters used for training DSP on various experimental settings. For the test learning rates, the first value corresponds to the learning rate for the Residual Adapter, while the second value corresponds to the learning rate for the pre-classifier transformation.}
\begin{tabular}{l|c|c|c|c}
    % \hline
    % Hyper-parameter & \multicolumn{3}{c|}{Multi-domain Setting} & Single-domain Setting \\
    \hline        \multirow{2}{*}{Setting} & Varying-Way   & Varying-Way & 5-Way      & Varying-Way\\
                    & Varying-Shot  & 5-Shot      & 1-Shot     & Varying-Shot\\
                    % & \multirow{2}{*}{Varying-Way Varying-Shot}  & \multirow{2}{*}{Varying-Way 5-Shot} & \multirow{2}{*}{5-Way 1-Shot} & \multirow{2}{*}{Varying-Way Varying-Shot} \\
    \hline
    Batch size & \multicolumn{3}{c|}{16} & 16\\
    Weight decay & \multicolumn{3}{c|}{0.0007} & 0.0007\\
    T-max & \multicolumn{3}{c|}{2500} & 2500\\
    Max iteration & \multicolumn{3}{c|}{160000} & 40000\\
    Initialization for $\mathbf{M}$ & \multicolumn{3}{c|}{$0.1\cdot\mathbf{I}$} & $0.1\cdot\mathbf{I}$\\
    Inner learning rate $\alpha_{\text{in}}$ & \multicolumn{3}{c|}{0.1} & 0.1\\
    Outer learning rate $\alpha_{\text{out}}$ & \multicolumn{3}{c|}{0.1} & 0.1\\
    The number of training inner-step & \multicolumn{3}{c|}{5} &  5\\
    The number of testing inner-step & \multicolumn{3}{c|}{40} &  40\\
    \hline
    Test learning rate $\beta$ for seen domain & (0.05, 0.30) & (0.05, 0.30) & (0.05, 0.30) & (0.05, 0.20)\\
    Test learning rate $\beta$ for unseen domain & (0.25, 0.05) & (0.25, 0.05) & (0.25, 0.05) & (0.25, 0.05) \\
    \hline
\end{tabular}
\label{tab:dsp_hyperparameters}
\end{table*}

\begin{table}
\centering
\setlength{\tabcolsep}{3.0mm}
\small
    \caption{Hyper-parameters used for training Dataset Classifier on various experimental settings. }
    \begin{tabular}{l|c}
         \hline
         Hyper-parameter \\
         \hline
         Batch size & 16 \\
         Weight decay & 0.0007 \\
         T-max & 500 \\
         Max iteration & 4000 \\
         Learning rate & 0.001 \\
         \hline
    \end{tabular}
    \label{tab:dataset_classifier_hyperparameters}
\end{table}

% \clearpage
\section{Additional Results}
\label{sec:additional_results}
\subsection{Varying-Way Five-Shot setting}
In the standard Meta-Dataset benchmark, tasks vary in the number of classes per task (`way') and the number of support images per class (`shot'), with `shot' ranging up to 100. Here, we evaluate TSP in Varying-Way Five-Shot setting, which poses a greater challenge due to the limited number of support images. The results in Table~\ref{tab:Varying-Way Five-Shot results} show that TSP\textsuperscript{\dag\dag} achieves top performance for 10 out of 13 datasets, including all 5 unseen datasets. Notably, TSP\textsuperscript{\dag\dag} demonstrates significantly higher scores in unseen domains compared to the previous best result ($+2.2\%$).

\subsection{Five-Way One-Shot setting}
We evaluate TSP under a more challenging setting, where only a single support image per class is available. 
As shown in Table~\ref{tab:Five_Way_One_Shot_results}, TSP\textsuperscript{\dag\dag} consistently outperforms the previous best results for 10 out of 13 datasets, while TSP\textsuperscript{\dag} also achieves the best or near-best results. 
% As shown in Table~\ref{tab:Five-Way One-Shot results}, TSP\textsuperscript{\dag} consistently outperforms the previous best results across all 13 datasets, while TSP\textsuperscript{\dag\dag} also achieves the best or near-best results. 
Notably, applying TSP to TA$^2$-Net results in a significant performance improvement in unseen domains ($+2.8\%$) compared to using TA$^2$-Net alone, demonstrating its efficacy in this highly challenging setting.

\begin{table}
\centering
\setlength{\tabcolsep}{2.0mm}
\small
\caption{Comparision of two DSP designs with and without the identity matrix. The DSP design $\mathbf{M^TM}$ does not guarantee positive definiteness.}
\label{tab:Additional_results_for_matrix}
\begin{tabular}{lcc|cc}
    \hline
                   Setting & \multicolumn{2}{c}{\begin{tabular}[c]{@{}c@{}}Varying-Way \\ Varying-Shot\end{tabular}} & \multicolumn{2}{|c}{\begin{tabular}[c]{@{}c@{}}Varying-Way \\ Five-Shot\end{tabular}}\\ \hline
                   DSP designs & $\mathbf{M^TM}$ & $\mathbf{M^TM+I}$ & $\mathbf{M^TM}$ & $\mathbf{M^TM+I}$\\ \hline
    Avg Seen   & 80.2 & \textbf{81.6} & 77.0 & \textbf{77.9} \\
    Avg Unseen & 68.2 & \textbf{79.8} & 63.3 & \textbf{73.7} \\
    Avg All    & 75.6 & \textbf{80.9} & 71.7 & \textbf{76.3} \\ \hline
\end{tabular}
\end{table}

% \clearpage
\section{Comparison between $\mathbf{M^TM+I}$ and $\mathbf{M^TM}$: Failure of DSP design $\mathbf{M^TM}$}
\label{sec:Additional_results_for_matrix}
In Section~\ref{sec:discussion}, we compared $\mathbf{LL^T+I}$ and $\mathbf{LL^T}$ to explore the impact of including the identity matrix in the DSP design. Extending this analysis, we now include two additional DSP designs, $\mathbf{M^TM+I}$ and $\mathbf{M^TM}$, for comprehensive examination, despite the latter's failure to meet positive definiteness. Our findings, presented in Table~\ref{tab:Additional_results_for_matrix}, demonstrate that the DSP design $\mathbf{M^TM+I}$ consistently outperforms $\mathbf{M^TM}$ in both Varying-Way Varying-Shot and Varying-Way Five-Shot settings. This contrasts with the results in Table~\ref{tab:precondition with identity matrix} of the manuscript, where including the identity matrix was effective in Varying-Way Varying-Shot setting. Furthermore, $\mathbf{M^TM}$ displays significantly lower performance compared to TSA~\cite{li2022cross}, which serves as the baseline for our method. 

To investigate the failure of the DSP design $\mathbf{M^TM}$, we compare the effective ranks~\cite{roy2007effective} of Task-Specific Preconditioners between two DSP designs, $\mathbf{M^TM}$ and $\mathbf{LL^T}$. The effective rank provides a numerical approximation of a matrix's rank, indicating the number of singular values distant from zero. Since positive definite matrices possess solely positive singular values, an effective rank significantly lower than the full rank implies that the preconditioners are far from positive definite. Table~\ref{tab:Additional_results_for_matrix_rank1} and Table~\ref{tab:Additional_results_for_matrix_rank2} present the averaged effective ranks for 17 Task-Specific Preconditioners of the DSP design $\mathbf{M^TM}$, differing only in settings: Varying-Way Varying-Shot and Varying-Way Five-Shot, respectively. Several preconditioners in these tables exhibit notably lower effective rank than the full rank, indicating their departure from positive definiteness. Consequently, these non-positive definite preconditioners may fail to determine the steepest descent direction in the parameter space, leading to the degraded performance observed in Table~\ref{tab:Additional_results_for_matrix}. In contrast, \ref{tab:Additional_results_for_matrix_rank3} and Table~\ref{tab:Additional_results_for_matrix_rank4} reveal that the averaged effective rank of 17 Task-Specific Preconditioners of the DSP design $\mathbf{LL^T}$ closely approach full rank. This finding aligns with the Cholesky factorization's assertion~\cite{horn2012matrix} that $\mathbf{LL^T}$ is positive definite, which confirms that Task-Specific Preconditioners constructed with $\mathbf{LL^T}$ are positive definite (Theorem~\ref{theorem: positive definiteness of TSP} holds as long as DSP $\mathbf{P}_k$ satisfies the positive definiteness). This observation corroborates the results presented in Table~\ref{tab:precondition with identity matrix} of the manuscript.

% \begin{tabular}[c]{@{}c@{}}Weight \\ Name\end{tabular}

\section{Time-efficiency of TSP compared to GAP}
\label{sec:time-efficiency of TSP compared to GAP}
Unlike GAP~\cite{kang2023meta}, which suffers from significant inference time due to time-intensive singular value decomposition~(SVD) calculations at every neural network layer during each inner-level training iteration, TSP achieves a much faster inference time. Specifically, GAP requires approximately 14.2 seconds per task, while TSP applied on TSA~\cite{li2022cross} completes inference in just 1.1 seconds--about 13 times faster than GAP--by leveraging pre-calculated DSPs and a dataset classifier to avoid time-intensive calculations during inference.\footnote{Inference time for GAP and TSP is measured on a single RTX3090 GPU and averaged over 100 test tasks.} This highlights the superior practical efficiency of TSP compared to the previous PGD-based method, GAP.

\section{Application Details for TSP\textsuperscript{\dag} and TSP\textsuperscript{\dag\dag}}
As shown Figure~\ref{fig:TSP_TA2Net}, we apply TSP to the state-of-the-art CDFSL methods, TSA~\cite{li2022cross} and TA$^2$-Net~\cite{guo2023task}. 
Figure~\ref{subfig:TSP_TSA_train} illustrates PGD with Domain-Specific Preconditioner~(DSP) applied to TSA during meta-training, where the DSP is selected based on the task's domain label, and PGD optimizes each task-specific parameter $\theta^l$ using the corresponding DSP. 
% Figure~\ref{subfig:TSP_TSA_test} shows PGD with a Task-Specific Preconditioner applied to TSA during meta-testing. In this case, each preconditioner is constructed from DSPs and task coefficients generated by the Dataset Classifier, and PGD optimizes each $\theta^l$ using the constructed preconditioner. 
Figure~\ref{subfig:TSP_TSA_test} shows PGD with a Task-Specific Preconditioner applied to TSA during meta-testing. In this case, each preconditioner is constructed by DSPs with task coefficients generated by the Dataset Classifier, and PGD optimizes each $\theta^l$ using the constructed preconditioner. 
Figure~\ref{subfig:TSP_TA2Net_train} depicts PGD with DSP applied to TA$^2$-Net during meta-training. Unlike TSA, which optimizes a single task-specific parameter for each module, TA$^2$-Net optimizes multiple task-specific parameters $\theta^{l,j}$. To accommodate this, we apply the same PGD to optimize the multiple parameters $\theta^{l,j}$. 
Figure~\ref{subfig:TSP_TA2Net_test} shows PGD with Task-Specific Preconditoner applied to TA$^2$-Net during meta-testing, where, similar to the meta-training phase, the same PGD is applied to optimize the multiple task-specific parameters $\theta^{l,j}$. 

\begin{table*}
\setlength{\tabcolsep}{2.0mm}
\small
\centering
\caption{Comparison to state-of-the-art methods in Varying-Way Five-Shot setting. Mean accuracy and 95$\%$ confidence interval are reported. The best results are highlighted in \textbf{bold}. TSP\textsuperscript{\dag} denotes TSP applied on TSA. TSP\textsuperscript{\dag\dag} denotes TSP applied on TA$^2$-Net.}
\label{tab:Varying-Way Five-Shot results}
\begin{tabular}{cccccccccc}
    \hline
    Test   Dataset   & \begin{tabular}[c]{@{}c@{}}Simple \\ CNAPS\end{tabular} & SUR & URT & URL & TSA& TA$^2$-Net & MOKD & TSP\textsuperscript{\dag}  & TSP\textsuperscript{\dag\dag} \\ \hline
    ImageNet         & 47.2$\pm$1.0 & 46.7$\pm$1.0 & 48.6$\pm$1.0 & 49.4$\pm$1.0 & 48.3$\pm$1.0 & 49.3$\pm$1.0 & 47.5$\pm$1.0 & 50.6$\pm$1.0 & \textbf{50.8$\pm$1.0} \\
    Omniglot         & 95.1$\pm$0.3 & 95.8$\pm$0.3 & 96.0$\pm$0.3 & 96.0$\pm$0.3 & 96.8$\pm$0.3 & 96.6$\pm$0.2 & 96.0$\pm$0.3 & \textbf{97.2$\pm$0.3} & 97.1$\pm$0.3 \\
    Aircraft         & 74.6$\pm$0.6 & 82.1$\pm$0.6 & 81.2$\pm$0.6 & 84.8$\pm$0.5 & 85.5$\pm$0.5 & 85.9$\pm$0.4 & 84.4$\pm$0.5 & 86.2$\pm$0.5 & \textbf{86.7$\pm$0.5} \\
    Birds            & 69.6$\pm$0.7 & 62.8$\pm$0.9 & 71.2$\pm$0.7 & 76.0$\pm$0.6 & 76.6$\pm$0.6 & 77.3$\pm$0.6 & 76.8$\pm$0.6 & 77.0$\pm$0.6 & \textbf{77.8$\pm$0.6} \\
    Textures         & 57.5$\pm$0.7 & 60.2$\pm$0.7 & 65.2$\pm$0.7 & 69.1$\pm$0.6 & 68.3$\pm$0.7 & 68.3$\pm$0.6 & 66.3$\pm$0.6 & 69.1$\pm$0.6 & \textbf{69.3$\pm$0.6}  \\
    Quick   Draw     & 70.9$\pm$0.6 & 79.0$\pm$0.5 & \textbf{79.2$\pm$0.5} & 78.2$\pm$0.5 & 77.9$\pm$0.6 & 78.5$\pm$0.5 & 78.9$\pm$0.5 & 78.7$\pm$0.6 & 78.8$\pm$0.6  \\
    Fungi            & 50.3$\pm$1.0 & 66.5$\pm$0.8 & 66.9$\pm$0.9 & 70.0$\pm$0.8 & 70.4$\pm$0.8 & 70.3$\pm$0.8 & 68.8$\pm$0.9 & \textbf{73.6$\pm$0.9}\; & 72.9$\pm$0.8  \\
    VGG   Flower     & 86.5$\pm$0.4 & 76.9$\pm$0.6 & 82.4$\pm$0.5 & 89.3$\pm$0.4 & 89.5$\pm$0.4 & 90.0$\pm$0.4 & 89.1$\pm$0.4 & 90.8$\pm$0.4\; & \textbf{91.1$\pm$0.4}  \\ \hline
    Traffic   Sign   & 55.2$\pm$0.8 & 44.9$\pm$0.9 & 45.1$\pm$0.9 & 57.5$\pm$0.8 & 72.3$\pm$0.6 & 76.7$\pm$0.5 & 59.2$\pm$0.8 & 79.6$\pm$0.5 & \textbf{80.9$\pm$0.5}  \\
    MSCOCO           & 49.2$\pm$0.8 & 48.1$\pm$0.9 & 52.3$\pm$0.9 & 56.1$\pm$0.8 & 56.0$\pm$0.8 & 56.0$\pm$0.8 & 51.8$\pm$0.8 & 57.5$\pm$0.9 & \textbf{59.3$\pm$0.8} \\
    MNIST            & 88.9$\pm$0.4 & 90.1$\pm$0.4 & 86.5$\pm$0.5 & 89.7$\pm$0.4 & 92.5$\pm$0.4 & 93.3$\pm$0.3 & 89.4$\pm$0.3 & 93.0$\pm$0.3 & \textbf{93.9$\pm$0.3}  \\
    CIFAR-10         & 66.1$\pm$0.7 & 50.3$\pm$1.0 & 61.4$\pm$0.7 & 66.0$\pm$0.7 & 72.0$\pm$0.7 & 73.1$\pm$0.7 & 58.8$\pm$0.7 & 73.5$\pm$0.7 & \textbf{74.2$\pm$0.7} \\
    CIFAR-100        & 53.8$\pm$0.9 & 46.4$\pm$0.9 & 52.5$\pm$0.9 & 57.0$\pm$0.9 & 64.1$\pm$0.8 & 64.1$\pm$0.8 & 55.3$\pm$0.9 & 65.0$\pm$0.8\; & \textbf{65.6$\pm$0.9} \\ \hline
    Avg   Seen   & 69.0 & 71.3 & 73.8 & 76.6 & 76.7 & 77.0 & 76.0 & 77.9 & \textbf{78.1}\\
    Avg   Unseen & 62.6 & 56.0 & 59.6 & 65.3 & 71.4 & 72.6 & 63.0 & 73.7 & \textbf{74.8}\\
    Avg   All    & 66.5 & 65.4 & 68.3 & 72.2 & 74.6 & 75.3 & 71.0 & 76.3 & \textbf{76.8}\\ \hline
    Avg   Rank   & 7.9  & 7.8  & 6.6  & 4.9  & 4.3  & 3.5  & 5.8 & 2.2 & \textbf{1.4}\\ \hline
\end{tabular}
\end{table*}

\begin{table*}[t]
\centering
\caption{Comparison to state-of-the-art methods in Five-Way One-Shot setting. Mean accuracy and 95$\%$ confidence interval are reported. The best results are highlighted in \textbf{bold}. TSP\textsuperscript{\dag} denotes TSP applied on TSA. TSP\textsuperscript{\dag\dag} denotes TSP applied on TA$^2$-Net.}
\label{tab:Five_Way_One_Shot_results}
\setlength{\tabcolsep}{2.0mm}
\small
\begin{tabular}{cccccccccc}
    \hline
    Test Dataset     & \begin{tabular}[c]{@{}c@{}}Simple \\ CNAPS\end{tabular} & SUR & URT & URL & TSA & TA$^2$-Net & MOKD & TSP\textsuperscript{\dag}  & TSP\textsuperscript{\dag\dag} \\ 
    % Test Dataset     & Simple CNAPS~\cite{bateni2020improved} & SUR~\cite{dvornik2020selecting} & URT~\cite{liu2020universal} & URL~\cite{li2021universal} & TSA~\cite{li2022cross} & TA$^2$-Net~\cite{guo2023task} & TSP\textsuperscript{\dag}  & TSP\textsuperscript{\dag\dag} \\ 
    \hline
    ImageNet         & 42.6$\pm$0.9 & 40.7$\pm$1.0 & 47.4$\pm$1.0 & 49.6$\pm$1.1 & 48.0$\pm$1.0 & 48.8$\pm$1.1 & 46.0$\pm$1.0 & 50.1$\pm$1.0 & \textbf{50.5$\pm$1.0} \\
    Omniglot         & 93.1$\pm$0.5 & 93.0$\pm$0.7 & 95.6$\pm$0.5 & 95.8$\pm$0.5 & 96.3$\pm$0.4 & 95.7$\pm$0.4 & 95.5$\pm$0.5 & 96.6$\pm$0.4 & \textbf{96.8$\pm$0.4} \\
    Aircraft         & 65.8$\pm$0.9 & 67.1$\pm$1.4 & 77.9$\pm$0.9 & 79.6$\pm$0.9 & 79.6$\pm$0.9 & 79.8$\pm$0.9 & 78.6$\pm$0.9 & 81.1$\pm$0.9 & \textbf{81.3$\pm$0.9} \\
    Birds            & 67.9$\pm$0.9 & 59.2$\pm$1.0 & 70.9$\pm$0.9 & 74.9$\pm$0.9 & 74.5$\pm$0.9 & 74.4$\pm$0.9 & 75.9$\pm$0.9 & \textbf{75.7$\pm$0.9} & \textbf{75.7$\pm$0.9} \\
    Textures         & 42.2$\pm$0.8 & 42.5$\pm$0.8 & 49.4$\pm$0.9 & 53.6$\pm$0.9 & 54.5$\pm$0.9 & 54.1$\pm$0.8 & 51.4$\pm$0.9 & \textbf{55.5$\pm$0.9} & 55.4$\pm$0.9 \\
    Quick Draw       & 70.5$\pm$0.9 & 79.8$\pm$0.9 & 79.6$\pm$0.9 & 79.0$\pm$0.8 & 79.3$\pm$0.9 & 78.9$\pm$0.9 & 78.9$\pm$0.9 & \textbf{80.7$\pm$0.8} & 80.2$\pm$0.9 \\
    Fungi            & 58.3$\pm$1.1 & 64.8$\pm$1.1 & 71.0$\pm$1.0 & 75.2$\pm$1.0 & 75.3$\pm$1.0 & 75.2$\pm$0.9 & 71.1$\pm$1.0 & 77.8$\pm$0.9 & \textbf{78.1$\pm$0.8} \\
    VGG Flower       & 79.9$\pm$0.7 & 65.0$\pm$1.0 & 72.7$\pm$0.0 & 79.9$\pm$0.8 & 80.3$\pm$0.8 & 80.1$\pm$0.8 & 79.8$\pm$0.8 & 81.0$\pm$0.8 & \textbf{81.1$\pm$0.8} \\ \hline
    Traffic Sign     & 55.3$\pm$0.9 & 44.6$\pm$0.9 & 52.7$\pm$0.9 & 57.9$\pm$0.9 & 57.2$\pm$1.0 & 54.1$\pm$1.0 & 57.0$\pm$0.9 & \textbf{57.4$\pm$1.0} & 56.9$\pm$1.0 \\
    MSCOCO           & 48.8$\pm$0.9 & 47.8$\pm$1.1 & 56.9$\pm$1.1 & 59.2$\pm$1.0 & 59.9$\pm$1.0 & 58.1$\pm$1.0 & 50.9$\pm$0.8 & 59.7$\pm$1.0 & \textbf{60.5$\pm$1.0} \\
    MNIST            & 80.1$\pm$0.9 & 77.1$\pm$0.9 & 75.6$\pm$0.9 & 78.7$\pm$0.9 & 80.1$\pm$0.9 & 80.3$\pm$0.9 & 72.5$\pm$0.9 & 81.4$\pm$0.8 & \textbf{81.7$\pm$0.9} \\
    CIFAR-10         & 50.3$\pm$0.9 & 35.8$\pm$0.8 & 47.3$\pm$0.9 & 54.7$\pm$0.9 & 55.8$\pm$0.9 & 52.9$\pm$1.0 & 47.3$\pm$0.8 & 55.9$\pm$0.9 & \textbf{56.0$\pm$0.9} \\
    CIFAR-100        & 53.8$\pm$0.9 & 42.9$\pm$1.0 & 54.9$\pm$1.1 & 61.8$\pm$1.0 & 63.7$\pm$1.0 & 61.0$\pm$1.1 & 60.2$\pm$1.0 & 65.2$\pm$1.0 & \textbf{65.6$\pm$1.0} \\ \hline
    Avg   Seen   & 65.0 & 64.0 & 70.6 & 73.5 & 73.5 & 73.4 & 72.2 & 74.8 & \textbf{74.9}\\
    Avg   Unseen & 57.7 & 49.6 & 57.5 & 62.5 & 63.3 & 61.3 & 57.5 & 63.9 & \textbf{64.1}\\
    Avg   All    & 62.2 & 58.5 & 65.5 & 69.2 & 69.6 & 68.7 & 66.5 & 70.6 & \textbf{70.8}\\ \hline
    Avg   Rank   & 7.5  & 8.2  & 6.8  & 4.2  & 3.5  & 4.8  & 6.2  & 1.9  & \textbf{1.5}\\ \hline 
\end{tabular}
\end{table*}

\begin{table*}
    \centering
    \setlength{\tabcolsep}{1.4mm}
    \small
    \caption{Averaged effective ranks for 17 Task-Specific Preconditioners of the DSP design $\mathbf{M^TM}$ in Varying-Way Varying-Shot setting. We average the effective ranks using 600 tasks randomly sampled from each domain. The left column denotes the name of each task-specific weight, while the right column indicates the full rank of each task-specific weight. 
    }
    \label{tab:Additional_results_for_matrix_rank1}
    \begin{tabular}{c|ccccccccccccc|c}
        \hline
         \begin{tabular}[c]{@{}c@{}}Weight's \\ Name\end{tabular} & \begin{tabular}[c]{@{}c@{}}Image \\ -Net\end{tabular} & \begin{tabular}[c]{@{}c@{}}Omni \\ -glot\end{tabular} & \begin{tabular}[c]{@{}c@{}}Airc \\ -raft\end{tabular} & Birds & \begin{tabular}[c]{@{}c@{}}Tex \\ -tures\end{tabular} & \begin{tabular}[c]{@{}c@{}}Quick \\ -Draw\end{tabular} & \begin{tabular}[c]{@{}c@{}}Fun \\ -gi\end{tabular} & \begin{tabular}[c]{@{}c@{}}VGG \\ Flower\end{tabular} & \begin{tabular}[c]{@{}c@{}}Traffic \\ Sign\end{tabular} & \begin{tabular}[c]{@{}c@{}}MS \\ -COCO\end{tabular} & \begin{tabular}[c]{@{}c@{}}MN \\ -IST\end{tabular} & \begin{tabular}[c]{@{}c@{}}CIFAR \\ -10\end{tabular} & \begin{tabular}[c]{@{}c@{}}CIFAR \\ -100\end{tabular} & \begin{tabular}[c]{@{}c@{}}Full \\ Rank\end{tabular} \\
        \hline
         layer1-0-$\alpha_1$ & 60.04 & 60.97 & 63.43 & 61.72 & 60.27 & 64.00 & 62.47 & 63.28 & 62.27 & 60.44 & 64.00 & 60.17 & 59.69 & 64\\
         layer1-0-$\alpha_2$ & 62.70 & 61.46 & 62.92 & 62.77 & 61.63 & 64.00 & 61.59 & 62.73 & 62.99 & 62.80 & 64.00 & 62.80 & 62.69 & 64\\
         layer1-1-$\alpha_1$ & 54.84 & 57.39 & 40.04 & 57.87 & 47.18 & 63.96 & 35.14 & 49.95 & 50.58 & 54.76 & 63.99 & 56.37 & 56.88 & 64\\
         layer1-1-$\alpha_2$ & 62.17 & 63.87 & 57.23 & 62.60 & 62.29 & 64.00 & 59.25 & 60.77 & 62.08 & 62.31 & 64.00 & 62.34 & 62.27 & 64\\
         layer2-0-$\alpha_1$ & 95.20 & 125.21 & 33.71 & 72.66 & 58.40 & 127.69 & 38.69 & 96.12 & 79.15 & 93.68 & 127.90 & 98.73 & 104.02 & 128\\
         layer2-0-$\alpha_2$ & 125.09 & 127.15 & 109.02 & 113.07 & 116.99 & 127.99 & 117.76 & 122.92 & 124.18 & 125.19 & 128.00 & 124.84 & 125.40 & 128\\
         layer2-1-$\alpha_1$ & 126.98 & 127.73 & 115.54 & 124.03 & 119.52 & 128.00 & 126.41 & 127.16 & 127.15 & 126.96 & 128.00 & 126.90 & 126.91 & 128\\
         layer2-1-$\alpha_2$ & 127.74 & 127.43 & 126.58 & 126.58 & 127.04 & 128.00 & 127.61 & 127.39 & 127.79 & 127.77 & 128.00 & 127.71 & 127.72 & 128\\
         layer3-0-$\alpha_1$ & 230.14 & 255.05 & 120.37 & 96.81 & 7.69 & 255.77 & 216.22 & 202.50 & 129.93 & 149.37 & 255.94 & 214.94 & 199.54 & 256\\
         layer3-0-$\alpha_2$ & 250.48 & 254.61 & 156.72 & 142.09 & 196.29 & 255.97 & 253.60 & 218.84 & 247.46 & 250.71 & 255.98 & 240.82 & 247.83 & 256\\
         layer3-1-$\alpha_1$ & 196.75 & 252.20 & 53.97 & 10.92 & 41.12 & 254.19 & 252.10 & 24.33 & 126.43 & 195.93 & 254.96 & 125.57 & 164.58 & 256\\
         layer3-1-$\alpha_2$ & 255.58 & 254.85 & 253.45 & 253.98 & 253.79 & 255.99 & 255.87 & 254.82 & 255.73 & 255.62 & 255.99 & 255.52 & 255.51 & 256\\
         layer4-0-$\alpha_1$ & 491.73 & 509.82 & 466.61 & 482.44 & 16.00 & 511.75 & 509.44 & 501.00 & 305.02 & 335.14 & 511.93 & 489.17 & 438.08 & 512\\
         layer4-0-$\alpha_2$ & 460.91 & 493.62 & 410.53 & 73.12 & 387.58 & 511.91 & 507.25 & 478.68 & 497.04 & 482.00 & 511.92 & 385.93 & 438.04 & 512\\
         layer4-1-$\alpha_1$ & 431.05 & 501.79 & 223.91 & 243.72 & 7.22 & 488.68 & 451.76 & 420.43 & 201.02 & 239.09 & 489.71 & 410.19 & 351.43 & 512\\
         layer4-1-$\alpha_2$ & 498.54 & 509.57 & 509.23 & 507.46 & 496.55 & 511.55 & 510.78 & 509.13 & 507.39 & 500.13 & 511.55 & 499.06 & 496.86 & 512\\
         pa-weight & 11.60 & 3.13 & 3.88 & 7.07 & 4.58 & 2.79 & 3.33 & 3.36 & 9.06 & 11.60 & 2.74 & 11.84 & 12.16 & 512\\
        \hline
    \end{tabular}
\end{table*}

\begin{table*}[t]
    \centering
    \setlength{\tabcolsep}{1.4mm}
    \small
    \caption{Averaged effective ranks for 17 Task-Specific Preconditioners of the DSP design $\mathbf{M^TM}$ in Varying-Way Five-Shot setting. We average the effective ranks using 600 tasks randomly sampled from each domain. The left column denotes the name of each task-specific weight, while the right column indicates the full rank of each task-specific weight.}
    \label{tab:Additional_results_for_matrix_rank2}
    \begin{tabular}{c|ccccccccccccc|c}
        \hline
         \begin{tabular}[c]{@{}c@{}}Weight's \\ Name\end{tabular} & \begin{tabular}[c]{@{}c@{}}Image \\ -Net\end{tabular} & \begin{tabular}[c]{@{}c@{}}Omni \\ -glot\end{tabular} & \begin{tabular}[c]{@{}c@{}}Airc \\ -raft\end{tabular} & Birds & \begin{tabular}[c]{@{}c@{}}Tex \\ -tures\end{tabular} & \begin{tabular}[c]{@{}c@{}}Quick \\ -Draw\end{tabular} & \begin{tabular}[c]{@{}c@{}}Fun \\ -gi\end{tabular} & \begin{tabular}[c]{@{}c@{}}VGG \\ Flower\end{tabular} & \begin{tabular}[c]{@{}c@{}}Traffic \\ Sign\end{tabular} & \begin{tabular}[c]{@{}c@{}}MS \\ -COCO\end{tabular} & \begin{tabular}[c]{@{}c@{}}MN \\ -IST\end{tabular} & \begin{tabular}[c]{@{}c@{}}CIFAR \\ -10\end{tabular} & \begin{tabular}[c]{@{}c@{}}CIFAR \\ -100\end{tabular} & \begin{tabular}[c]{@{}c@{}}Full \\ Rank\end{tabular} \\
        \hline
         layer1-0-$\alpha_1$ &60.14 &61.04 &63.43 &61.70 &60.49 &64.00 &62.46 &63.28 &62.00 &60.37 &64.00 &60.31 &59.76 &64 \\
         layer1-0-$\alpha_2$ &62.70 &61.52 &62.93 &62.79 &61.86 &64.00 &61.61 &62.73 &62.98 &62.80 &64.00 &62.80 &62.70 &64\\
         layer1-1-$\alpha_1$ &54.41 &57.46 &40.18 &57.88 &48.46 &63.96 &35.38 &49.97 &51.36 &55.34 &63.99 &55.87 &56.69 &64\\
         layer1-1-$\alpha_2$ &62.13 &63.86 &57.28 &62.61 &62.40 &64.00 &59.31 &60.78 &62.15 &62.36 &64.00 &62.31 &62.27 &64\\
         layer2-0-$\alpha_1$ &93.59 &124.77 &34.00 &73.67 &62.76 &127.68 &39.19 &96.04 &81.61 &95.96 &127.90 &96.43 &103.21 &128\\
         layer2-0-$\alpha_2$ &124.92 &127.16 &109.17 &113.52 &118.17 &127.99 &117.88 &122.93 &124.42 &125.25 &128.00 &124.63 &125.37 &128\\
         layer2-1-$\alpha_1$ &126.98 &127.74 &115.65 &124.15 &120.53 &128.00 &126.43 &127.16 &127.15 &126.91 &128.00 &126.89 &126.92 &128\\
         layer2-1-$\alpha_2$ &127.74 &127.45 &126.59 &126.63 &127.16 &128.00 &127.61 &127.39 &127.80 &127.76 &128.00 &127.71 &127.73 &128\\
         layer3-0-$\alpha_1$ &229.05 &254.92 &121.04 &99.77 &9.22 &255.76 &216.45 &202.70 &139.21 &136.38 &255.94 &211.75 &200.61 &256\\
         layer3-0-$\alpha_2$ &249.93 &254.62 &157.35 &144.87 &202.20 &255.97 &253.61 &219.02 &248.38 &249.63 &255.98 &239.71 &248.04 &256\\
         layer3-1-$\alpha_1$ &193.29 &251.15 &54.49 &11.90 &47.84 &254.12 &251.87 &24.43 &138.76 &184.25 &254.95 &127.65 &166.39 & 256\\
         layer3-1-$\alpha_2$ &255.58 &254.88 &253.48 &254.05 &254.07 &255.99 &255.87 &254.83 &255.72 &255.59 &255.99 &255.53 &255.52 &256\\
         layer4-0-$\alpha_1$ &492.50 &509.86 &466.98 &482.91 &19.81 &511.75 &509.23 &501.06 &322.40 &308.72 &511.93 &487.09 &439.86 &512\\
         layer4-0-$\alpha_2$ &457.24 &493.82 &411.01 &77.99 &398.85 &511.91 &507.03 &478.88 &496.18 &476.10 &511.92 &382.92 &440.78 &512\\
         layer4-1-$\alpha_1$ &431.12 &501.57 &225.25 &248.87 &8.84 &488.65 &451.73 &420.80 &220.69 &213.26 &489.71 &405.75 &354.55 & 512\\
         layer4-1-$\alpha_2$ &499.01 &509.62 &509.24 &507.28 &497.76 &511.55 &510.73 &509.15 &506.38 &499.75 &511.55 &499.69 &497.14 &512\\
         pa-weight &11.43 &3.21 &3.95 &7.29 &5.60 &2.79 &3.45 &3.39 &9.62 &11.74 &2.74 &11.65 &12.10 &512\\
        \hline
    \end{tabular}
\end{table*}

\begin{table*}[t]
    \centering
    \setlength{\tabcolsep}{1.4mm}
    \small
    \caption{
        Averaged effective ranks for 17 Task-Specific Preconditioners of the DSP design $\mathbf{LL^T}$ in Varying-Way Varying-Shot setting. We average the effective ranks using 600 tasks randomly sampled from each domain. The left column denotes the name of each task-specific weight, while the right column indicates the full rank of each task-specific weight.
    }
    \label{tab:Additional_results_for_matrix_rank3}
    \begin{tabular}{c|ccccccccccccc|c}
        \hline
         \begin{tabular}[c]{@{}c@{}}Weight's \\ Name\end{tabular} & \begin{tabular}[c]{@{}c@{}}Image \\ -Net\end{tabular} & \begin{tabular}[c]{@{}c@{}}Omni \\ -glot\end{tabular} & \begin{tabular}[c]{@{}c@{}}Airc \\ -raft\end{tabular} & Birds & \begin{tabular}[c]{@{}c@{}}Tex \\ -tures\end{tabular} & \begin{tabular}[c]{@{}c@{}}Quick \\ -Draw\end{tabular} & \begin{tabular}[c]{@{}c@{}}Fun \\ -gi\end{tabular} & \begin{tabular}[c]{@{}c@{}}VGG \\ Flower\end{tabular} & \begin{tabular}[c]{@{}c@{}}Traffic \\ Sign\end{tabular} & \begin{tabular}[c]{@{}c@{}}MS \\ -COCO\end{tabular} & \begin{tabular}[c]{@{}c@{}}MN \\ -IST\end{tabular} & \begin{tabular}[c]{@{}c@{}}CIFAR \\ -10\end{tabular} & \begin{tabular}[c]{@{}c@{}}CIFAR \\ -100\end{tabular} & \begin{tabular}[c]{@{}c@{}}Full \\ Rank\end{tabular} \\
        \hline
         layer1-0-$\alpha_1$ &63.97 &63.95 &63.99 &63.98 &63.99 &64.00 &63.99 &64.00 &63.99 &63.98 &64.00 &63.96 &63.97 &64 \\
         layer1-0-$\alpha_2$ &63.97 &63.96 &64.00 &63.99 &64.00 &64.00 &64.00 &64.00 &63.99 &63.98 &64.00 &63.97 &63.97 &64\\
         layer1-1-$\alpha_1$ &63.92 &63.90 &63.99 &63.96 &63.99 &64.00 &63.99 &63.99 &63.97 &63.95 &64.00 &63.91 &63.92 &64\\
         layer1-1-$\alpha_2$ &63.97 &63.99 &63.99 &63.99 &64.00 &64.00 &64.00 &64.00 &63.99 &63.98 &64.00 &63.97 &63.97 &64\\
         layer2-0-$\alpha_1$ &127.87 &127.97 &127.97 &127.92 &127.96 &128.00 &127.98 &127.98 &127.96 &127.91 &128.00 &127.84 &127.86 &128\\
         layer2-0-$\alpha_2$ &127.97 &127.99 &127.99 &127.98 &127.98 &128.00 &127.99 &127.99 &127.99 &127.98 &128.00 &127.97 &127.97 & 128\\
         layer2-1-$\alpha_1$ &127.99 &128.00 &127.99 &127.99 &127.98 &128.00 &128.00 &128.00 &128.00 &127.99 &128.00 &127.99 &127.99 &128\\
         layer2-1-$\alpha_2$ &127.99 &128.00 &128.00 &128.00 &128.00 &128.00 &128.00 &128.00 &128.00 &128.00 &128.00 &127.99 &127.99 &128\\
         layer3-0-$\alpha_1$ &255.95 &256.00 &255.96 &255.94 &255.66 &256.00 &255.97 &255.98 &255.98 &255.96 &256.00 &255.94 &255.95 &256\\
         layer3-0-$\alpha_2$ &255.98 &256.00 &255.95 &255.96 &255.90 &256.00 &255.99 &255.99 &255.99 &255.98 &256.00 &255.97 &255.98 &256\\
         layer3-1-$\alpha_1$ &255.99 &256.00 &255.93 &255.91 &255.59 &256.00 &255.98 &255.97 &255.99 &255.98 &256.00 &255.99 &255.99 &256\\
         layer3-1-$\alpha_2$ &256.00 &256.00 &255.98 &255.99 &255.99 &256.00 &256.00 &255.99 &256.00 &256.00 &256.00 &255.99 &256.00 &256\\
         layer4-0-$\alpha_1$ &511.97 &512.00 &511.91 &511.95 &510.54 &512.00 &511.95 &511.97 &511.97 &511.95 &512.00 &511.97 &511.96 &512\\
         layer4-0-$\alpha_2$ &511.96 &511.98 &511.79 &511.81 &511.66 &512.00 &511.90 &511.92 &511.95 &511.96 &512.00 &511.96 &511.97 &512\\
         layer4-1-$\alpha_1$ &511.90 &511.97 &511.54 &511.78 &507.13 &511.92 &511.58 &511.84 &511.77 &511.76 &511.93 &511.91 &511.89 &512\\
         layer4-1-$\alpha_2$ &511.97 &511.99 &511.97 &511.97 &511.81 &512.00 &511.97 &511.98 &511.98 &511.98 &512.00 &511.97 &511.97 &512\\
         pa-weight &494.46 &481.54 &494.69 &495.60 &490.04 &486.75 &484.23 &493.87 &490.07 &493.48 &486.75 &495.33 &495.25 &512\\
        \hline
    \end{tabular}
\end{table*}

\begin{table*}[t]
    \centering
    \setlength{\tabcolsep}{1.4mm}
    \small
    \caption{Averaged effective ranks for 17 Task-Specific Preconditioners of the DSP design $\mathbf{LL^T}$ in Varying-Way Five-Shot setting. We average the effective ranks using 600 tasks randomly sampled from each domain. The left column denotes the name of each task-specific weight, while the right column indicates the full rank of each task-specific weight.}
    \label{tab:Additional_results_for_matrix_rank4}
    \begin{tabular}{c|ccccccccccccc|c}
        \hline
         \begin{tabular}[c]{@{}c@{}}Weight's \\ Name\end{tabular} & \begin{tabular}[c]{@{}c@{}}Image \\ -Net\end{tabular} & \begin{tabular}[c]{@{}c@{}}Omni \\ -glot\end{tabular} & \begin{tabular}[c]{@{}c@{}}Airc \\ -raft\end{tabular} & Birds & \begin{tabular}[c]{@{}c@{}}Tex \\ -tures\end{tabular} & \begin{tabular}[c]{@{}c@{}}Quick \\ -Draw\end{tabular} & \begin{tabular}[c]{@{}c@{}}Fun \\ -gi\end{tabular} & \begin{tabular}[c]{@{}c@{}}VGG \\ Flower\end{tabular} & \begin{tabular}[c]{@{}c@{}}Traffic \\ Sign\end{tabular} & \begin{tabular}[c]{@{}c@{}}MS \\ -COCO\end{tabular} & \begin{tabular}[c]{@{}c@{}}MN \\ -IST\end{tabular} & \begin{tabular}[c]{@{}c@{}}CIFAR \\ -10\end{tabular} & \begin{tabular}[c]{@{}c@{}}CIFAR \\ -100\end{tabular} & \begin{tabular}[c]{@{}c@{}}Full \\ Rank\end{tabular} \\
        \hline
         layer1-0-$\alpha_1$ &63.97 &63.95 &63.99 &63.98 &63.99 &64.00 &63.99 &64.00 &63.99 &63.98 &64.00 &63.96 &63.97 &64\\
         layer1-0-$\alpha_2$ &63.97 &63.96 &64.00 &63.99 &64.00 &64.00 &64.00 &64.00 &63.99 &63.98 &64.00 &63.97 &63.97 &64\\
         layer1-1-$\alpha_1$ &63.92 &63.90 &63.99 &63.96 &63.99 &64.00 &63.99 &63.99 &63.97 &63.94 &64.00 &63.91 &63.92 &64\\
         layer1-1-$\alpha_2$ &63.97 &63.99 &63.99 &63.99 &64.00 &64.00 &64.00 &64.00 &63.99 &63.98 &64.00 &63.97 &63.97 &64\\
         layer2-0-$\alpha_1$ &127.87 &127.97 &127.97 &127.93 &127.96 &128.00 &127.98 &127.98 &127.95 &127.91 &128.00 &127.85 &127.86 &128\\
         layer2-0-$\alpha_2$ &127.97 &127.99 &127.99 &127.98 &127.99 &128.00 &127.99 &127.99 &127.99 &127.98 &128.00 &127.97 &127.97 &128\\
         layer2-1-$\alpha_1$ &127.99 &128.00 &127.99 &127.99 &127.99 &128.00 &128.00 &128.00 &128.00 &127.99 &128.00 &127.99 &127.99 &128\\
         layer2-1-$\alpha_2$ &127.99 &128.00 &128.00 &128.00 &128.00 &128.00 &128.00 &128.00 &128.00 &128.00 &128.00 &127.99 &127.99 &128\\
         layer3-0-$\alpha_1$ &255.95 &256.00 &255.97 &255.94 &255.78 &256.00 &255.97 &255.98 &255.98 &255.96 &256.00 &255.94 &255.95 & 256\\
         layer3-0-$\alpha_2$ &255.98 &256.00 &255.95 &255.95 &255.94 &256.00 &255.99 &255.99 &255.99 &255.98 &256.00 &255.97 &255.98 &256\\
         layer3-1-$\alpha_1$ &255.99 &256.00 &255.93 &255.90 &255.74 &256.00 &255.98 &255.97 &255.99 &255.98 &256.00 &255.99 &255.99 &256\\
         layer3-1-$\alpha_2$ &256.00 &256.00 &255.98 &255.99 &255.99 &256.00 &256.00 &256.00 &256.00 &256.00 &256.00 &256.00 &256.00 &256\\
         layer4-0-$\alpha_1$ &511.97 &512.00 &511.91 &511.95 &511.06 &512.00 &511.95 &511.97 &511.97 &511.94 &512.00 &511.97 &511.97 &512\\
         layer4-0-$\alpha_2$ &511.96 &511.98 &511.80 &511.79 &511.78 &512.00 &511.90 &511.92 &511.96 &511.96 &512.00 &511.96 &511.97 &512\\
         layer4-1-$\alpha_1$ &511.89 &511.97 &511.55 &511.76 &508.74 &511.92 &511.58 &511.84 &511.79 &511.73 &511.92 &511.91 &511.89 &512\\
         layer4-1-$\alpha_2$ &511.97 &511.99 &511.97 &511.97 &511.87 &512.00 &511.97 &511.98 &511.98 &511.97 &512.00 &511.97 &511.97 &512\\
         pa-weight &494.37 &481.54 &494.72 &495.46 &490.94 &486.75 &484.26 &493.87 &490.64 &494.00 &486.75 &495.13 &495.09 &512\\
        \hline
    \end{tabular}
\end{table*}

\begin{figure*}[t]
    \centering
    \begin{subfigure}[t]{0.7\textwidth}
        \includegraphics[width=\textwidth]{Figures/TSP_train.png}
        \subcaption{}
        \label{subfig:TSP_TSA_train}
    \end{subfigure}
    \begin{subfigure}[t]{0.7\textwidth}
        \includegraphics[width=\textwidth]{Figures/TSP_test.png}
        \subcaption{}
        \label{subfig:TSP_TSA_test}
    \end{subfigure}
    \begin{subfigure}[t]{0.7\textwidth}
        \includegraphics[width=\textwidth]{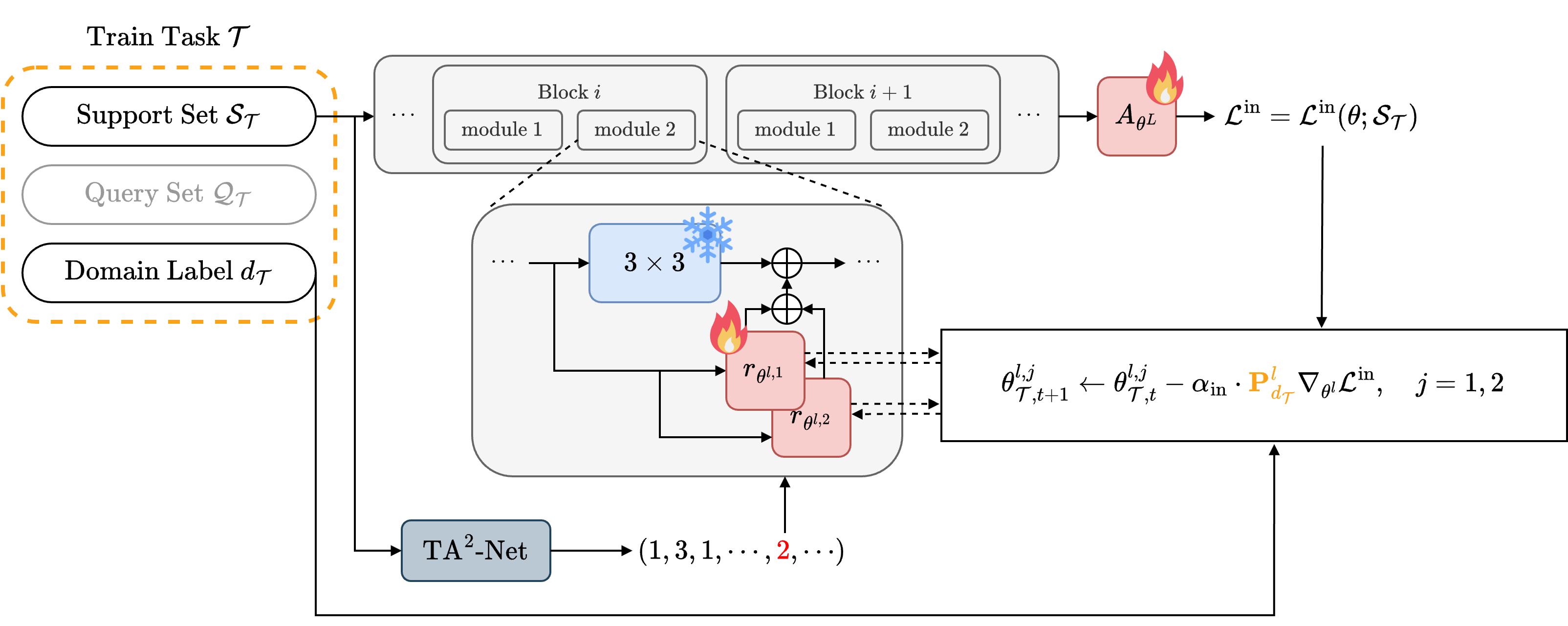}
        \subcaption{}
        \label{subfig:TSP_TA2Net_train}
    \end{subfigure}
    \begin{subfigure}[t]{0.7\textwidth}
        \includegraphics[width=\textwidth]{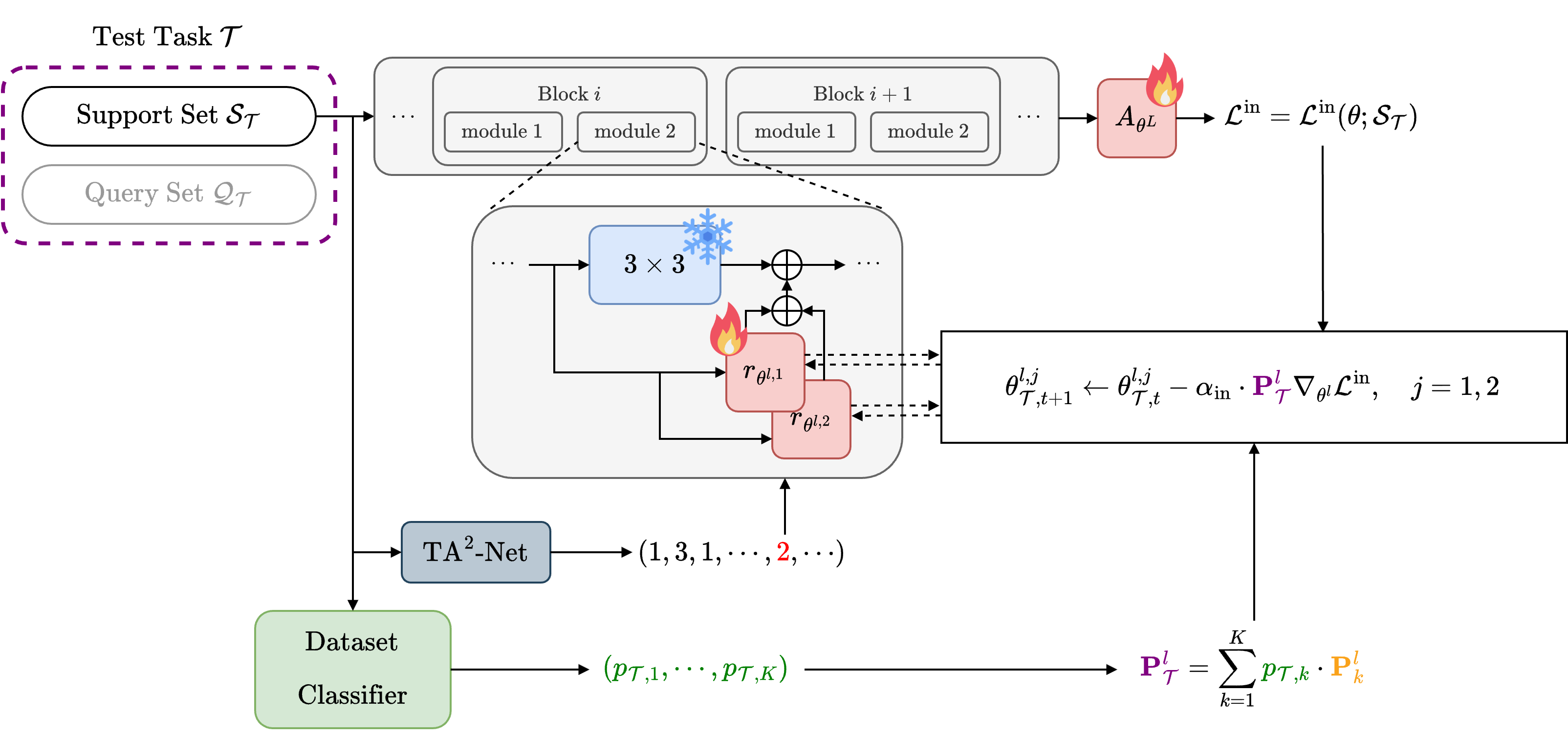}
        \subcaption{}
        \label{subfig:TSP_TA2Net_test}
    \end{subfigure}
    \caption{TSP applied on TSA and TA$^2$-Net. (a) PGD with Domain-Specific Preconditioner~(DSP) applied on TSA during meta-training. (b) PGD with Task-Specific Preconditioner applied on TSA during meta-testing. (c) PGD with Domain-Specific Preconditioner~(DSP) applied on TA$^2$-Net during meta-training. (d) PGD with Task-Specific Preconditioner applied on TA$^2$-Net during meta-testing.}
    \label{fig:TSP_TA2Net}
\end{figure*}